\ifdefined\jmlrformat
  \documentclass[twoside,11pt]{article}
  \usepackage{jmlr2e}
	\usepackage[pres]{lwm} %
	\usepackage{natbib}
\else
	\documentclass[10pt]{article}
	\usepackage{lwm} %
	\usepackage[numbers]{natbib}
\fi

\usepackage{multirow}
\usepackage{url}
\usepackage{epsfig}
\usepackage{verbatim}
\usepackage{algorithm}
\usepackage{algorithmic}
\usepackage[hyperindex]{hyperref} %
\hypersetup{
    colorlinks=true,       %
    linkcolor=black,          %
    citecolor=black,        %
    filecolor=black,      %
    urlcolor=black           %
}
\usepackage{macros}

\graphicspath{{figs/}}

\ifdefined\jmlrformat
\jmlrheading{}{}{}{}{}{Lester Mackey, Ameet Talwalkar and Michael I. Jordan}
\ShortHeadings{Distributed Matrix Completion and Robust Factorization}{Mackey, Talwalkar and Jordan}
\firstpageno{1}
\else
\usepackage{ourarxiv,times} %
\nipsfinalcopy 
\title{Distributed Matrix Completion\\ and Robust Factorization}
\author{
Lester Mackey\textsuperscript{a, $\dagger$}\\
\And
Ameet Talwalkar\textsuperscript{b, $\dagger$}\\
\And
Michael I. Jordan\textsuperscript{b, c}\\
}
\fi

\begin{document}

\ifdefined\jmlrformat
\title{Distributed Matrix Completion and Robust Factorization}

\author{\name Lester Mackey\textsuperscript{$\dagger$} \email lmackey@stanford.edu \\
       \addr Stanford University\\
       Department of Statistics \\
       390 Serra Mall \\
       Stanford, CA 94305
       \AND
       \name Ameet Talwalkar\textsuperscript{$\dagger$} \email ameet@cs.berkeley.edu \\
       \addr University of California, Berkeley\\
       Department of Electrical Engineering and Computer Science\\
       465 Soda Hall\\
       Berkeley, CA 94720
       \AND
       \name Michael I. Jordan \email jordan@cs.berkeley.edu \\
       \addr University of California, Berkeley\\
       Department of Electrical Engineering and Computer Science and Department of Statistics\\
       465 Soda Hall\\
       Berkeley, CA 94720
}

\editor{TBD}

\maketitle
{
\vspace*{-0.8cm}
\noindent\textsuperscript{$\dagger$} These authors contributed equally.\par\bigskip
\par
}
\else
\maketitle
{\centering
\vspace*{-0.8cm}
\noindent\textsuperscript{a} Department of Statistics, Stanford \\ 
\textsuperscript{b} Department of Engineering and Computer Science, UC Berkeley \\
\textsuperscript{c} Department of Statistics, UC Berkeley \\
\textsuperscript{$\dagger$} These authors contributed equally.\par\bigskip
\par
}
\fi

\begin{abstract}%
If learning methods are to scale to the massive sizes of modern datasets,
it is essential for the field of machine learning to embrace parallel and distributed computing.
Inspired by the recent
development of matrix factorization methods with rich theory
but poor computational complexity and by the relative
ease of mapping matrices onto distributed architectures, we introduce 
a scalable divide-and-conquer framework for noisy matrix factorization.  
We present a thorough theoretical analysis of this framework in 
which we characterize the statistical errors introduced by the ``divide'' 
step and control their magnitude in the ``conquer'' step, so that the 
overall algorithm enjoys high-probability estimation guarantees comparable 
to those of its base algorithm.  We also present experiments in collaborative 
filtering and video background modeling that demonstrate the near-linear 
to superlinear speed-ups attainable with this approach.
\end{abstract}

\ifdefined\jmlrformat
\fi

\section{Introduction}
The scale of modern scientific and technological datasets poses major new 
challenges for computational and statistical science.  
Data analyses and learning algorithms suitable
for modest-sized datasets are often entirely infeasible for the terabyte 
and petabyte datasets that are fast becoming the norm.  There are two basic 
responses to this challenge.  One response is to abandon algorithms 
that have superlinear complexity, focusing attention on simplified 
algorithms that---in the setting of massive data---may achieve satisfactory
results because of the statistical strength of the data.  While this
is a reasonable research strategy, it requires developing suites of 
algorithms of varying computational complexity for each inferential 
task and calibrating statistical and computational efficiencies.  
There are many open problems that need to be solved if such an effort
is to bear fruit.

The other response to the massive data problem is to retain existing
algorithms but to apply them to subsets of the data.  To obtain useful
results under this approach, one embraces parallel and distributed 
computing architectures, applying existing base algorithms to multiple 
subsets of the data in parallel and then combining the results.  Such a divide-and-conquer 
methodology has two main virtues: (1) it builds directly on algorithms 
that have proven their value at smaller scales
and that often have strong theoretical guarantees, 
and (2) it requires little in the way of new algorithmic 
development.  The major challenge, however, is in preserving the theoretical 
guarantees of the base algorithm once one embeds the algorithm in a 
computationally-motivated divide-and-conquer procedure.  Indeed, the 
theoretical guarantees often refer to subtle statistical properties of 
the data-generating mechanism (e.g., sparsity, information spread,
and near low-rankedness).  These may or may not be retained under the ``divide'' 
step of a putative divide-and-conquer solution.  In fact, we generally 
would expect subsampling operations to damage the relevant statistical 
structures. 
Even if these properties are preserved, we face the difficulty of 
combining the intermediary results of the ``divide'' step
into a final consilient solution to the original problem.
The question, therefore, is whether we can design divide-and-conquer 
algorithms that manage the tradeoffs relating these statistical properties 
to the computational degrees of freedom such that the overall algorithm 
provides a scalable solution that retains the theoretical guarantees of 
the base algorithm.

In this paper,\footnote{A preliminary form of this work appears in \citet{MackeyTaJo11}.} 
we explore this issue in the context of an important class of
machine learning algorithms---the matrix factorization algorithms underlying a
wide variety of practical applications, including collaborative filtering for
recommender systems~(e.g., \citep{KorenBeVo09} and the references therein),
link prediction for social networks~\citep{Hoff05}, click prediction for web
search~\citep{DasDaGaRa07}, video surveillance~\citep{CandesLiMaWr09},
graphical model selection~\citep{ChandrasekaranSaPaWi09}, document
modeling~\citep{MinZhWrMa10}, and image alignment~\citep{PengGaWrXuMa10}.  We
focus on two instances of the general matrix factorization problem: noisy
matrix completion~\citep{CandesPl10}, where the goal is to recover a low-rank
matrix from a small subset of noisy entries, and noisy robust matrix
factorization~\citep{CandesLiMaWr09,ChandrasekaranSaPaWi09}, where the aim is
to recover a low-rank matrix from corruption by noise and outliers of arbitrary
magnitude.  These two classes of matrix factorization problems have attracted
significant interest in the research community. 

Various approaches have been proposed for scalable noisy matrix factorization
problems, in particular for noisy matrix completion, though the vast majority
tackle rank-constrained non-convex formulations of these problems with no
assurance of finding optimal
solutions~\citep{ZhouWiScPa08,GemullaNiHaSi11,RechtRe11,NiuReReWr11,YuHsSiDh12}.  In
contrast, convex formulations of noisy matrix factorization relying on the
nuclear norm have been shown to admit strong theoretical estimation
guarantees~\citep{AgarwalNeWa11,CandesLiMaWr09,CandesPl10,NegahbanWa12}, and a
variety of algorithms~\citep[e.g.,][]{LinGaWrWuChMa09,MaGoCh11,TohYu10} have
been developed for solving both matrix completion and robust matrix
factorization via convex relaxation.  Unfortunately, however, all of these
methods are inherently sequential, and all rely on the repeated and costly
computation of truncated singular value decompositions (SVDs), factors that
severely limit the scalability of the algorithms. Moreover, previous attempts
at reducing this computational burden have introduced approximations without
theoretical justification~\citep{MuDoYuYa11}.

To address this key problem of noisy matrix factorization in a scalable and
theoretically sound manner, we propose a divide-and-conquer framework for
large-scale matrix factorization. Our framework, entitled Divide-Factor-Combine
(\fastmf), randomly divides the original matrix factorization task into cheaper
subproblems, solves those subproblems in parallel using a base matrix
factorization algorithm for nuclear norm regularized formulations, and combines
the solutions to the subproblems using efficient techniques from randomized
matrix approximation.  We develop a thoroughgoing theoretical analysis for the
\fastmf framework, linking statistical properties of the underlying matrix to
computational choices in the algorithms and thereby providing conditions under
which statistical estimation of the underlying matrix is possible.  We also
present experimental results for several \fastmf variants demonstrating that
\fastmf can provide near-linear to superlinear speed-ups in practice.

The remainder of the paper is organized as follows. In
Sec.~\ref{sec:fast_algs}, we define the setting of noisy matrix factorization
and introduce the components of the \fastmf framework.
Secs.~\ref{sec:roadmap}, \ref{sec:theory}, and \ref{sec:theory-spike}
present our theoretical analysis of \fastmf, along with a 
new analysis of convex noisy matrix completion and 
a novel characterization of randomized matrix approximation algorithms.
To illustrate the practical speed-up and robustness of \fastmf, 
we present experimental results on collaborative filtering, 
video background modeling, and simulated
data in Sec.~\ref{sec:experiments}. 
Finally, we conclude in Sec.~\ref{sec:conclusion}.

\paragraph{Notation}
For a matrix $\M \in \reals^{m \times n}$, 
we define $\M_{(i)}$ as the $i$th row vector, $\M^{(j)}$ as the $j$th column vector,
and $\M_{ij}$ as the $ij$th entry. 
If $\rank{\M}=r$, we write the compact singular value decomposition (SVD) of
$\M$ as $\U_{M} \mSigma_{M} \V_{M}^\top$, where $\mSigma_M$ is diagonal and
contains the $r$ non-zero singular values of $\M$, and $\U_M \in \reals^{m
\times r}$ and $\V_M \in \reals^{n \times r}$ are the corresponding left and
right singular vectors of $\M$.  We define $\pinv{\M} =\V_{M} \mSigma_{M}^{-1}
\U_{M}^\top$ as the Moore-Penrose pseudoinverse of $\M$ and $\P_M=\M\pinv{\M}$
as the orthogonal projection onto the column space of $\M$.  We let $\norm{
\cdot }_2$, $\norm{ \cdot }_F$, and $\norm{\cdot}_*$ respectively denote the
spectral, Frobenius, and nuclear norms of a matrix, 
$\norm{\cdot}_\infty$ denote the maximum entry of a matrix,
and $\norm{\cdot}$ represent the $\ell_2$ norm of a vector.

\section{The Divide-Factor-Combine Framework}
\label{sec:fast_algs}
In this section, we present a general divide-and-conquer framework for scalable noisy
matrix factorization.  We begin by defining the problem setting of interest.
\subsection{Noisy Matrix Factorization (MF)}
In the setting of noisy matrix factorization, we observe a subset of the
entries of a matrix $\M = \L_0 + \S_0 + \Z_0 \in\reals^{m\times n}$, where
$\L_0$ has rank $r\ll m, n$, $\S_0$ represents a sparse matrix of outliers
of arbitrary magnitude, and $\Z_0$ is a dense noise matrix.  We let $\obsset$
represent the locations of the observed entries and $\obsproj$ be the orthogonal
projection onto the space of $m \times n$ matrices with support $\obsset$, so
that $$(\obsproj(\M))_{ij} = \M_{ij},\ \text{if}\ (i,j)\in\Omega\quad
\text{and}\quad (\obsproj(\M))_{ij} = 0\ \text{otherwise}.%
\footnote{When $\Q$ is a submatrix of $\M$ we abuse notation
and let $\obsproj(\Q)$ be the corresponding submatrix of
$\obsproj(\M)$.}$$ 
Our goal is to
estimate the low-rank matrix $\L_0$ from $\obsproj(\M)$ with error proportional
to the noise level $\Delta\defeq\norm{\Z_0}_F$.  We will focus on two specific
instances of this general problem:
\begin{itemize}
\item{\bf Noisy Matrix Completion (MC):} $s\defeq|\obsset|$ entries of $\M$ are
revealed uniformly without replacement, along with their locations.  There are
no outliers, so that $\S_0$ is identically zero.
\item{\bf Noisy Robust Matrix Factorization (RMF):} $\S_0$ is identically zero
save for $s$ outlier entries of arbitrary magnitude with unknown locations
distributed uniformly without replacement.  All entries of $\M$ are observed,
so that $\obsproj(\M)=\M$.
\end{itemize}

\subsection{Divide-Factor-Combine}

The Divide-Factor-Combine (\fastmf) framework divides the expensive task of matrix
factorization into smaller subproblems, executes those subproblems in parallel, and then efficiently
combines the results into a final low-rank estimate of $\L_0$.  
We highlight three variants of this general framework in Algorithms \ref{alg:fastmf-proj}, \ref{alg:fastmf-rp}, and \ref{alg:fastmf-nys}. 
These algorithms, which we refer to as \projmf, \rpmf, and \gnysmf, differ in
their strategies for division and recombination
but adhere to a common pattern of three simple steps:
\ifdefined\jacmformat
\begin{describe}{{\bf(D step)}}
\else
\begin{description}%
\fi
\item[{\bf(D step)}] \textbf{Divide input matrix into submatrices:} \projmf and \rpmf randomly
partition $\obsproj(\M)$ into $t$ $l$-column submatrices,
$\{\obsproj(\C_1),\ldots,\obsproj(\C_t)\}$,\footnote{For ease of discussion, we
assume that $t$ evenly divides $n$ so that $l = n /t$.  
In general, $\obsproj(\M)$ can always be partitioned into $t$
submatrices, each with either $\lfloor n / t \rfloor$ or $\lceil n / t \rceil$
columns.} while \gnysmf selects an $l$-column submatrix, $\obsproj(\C)$, and a
$d$-row submatrix, $\obsproj(\R)$, uniformly at random.

\item[{\bf(F step)}] \textbf{Factor each submatrix in parallel using any base MF algorithm:} \projmf
and \rpmf perform $t$ parallel submatrix factorizations, while \gnysmf performs two such
parallel factorizations.  Standard base MF algorithms output the following low-rank
approximations: $\{\hat \C_1,\ldots,\hat \C_t\}$ for \projmf and \rpmf; $\hat \C$ 
and $\hat \R$ for \gnysmf.  All matrices are retained in factored form.

\item[{\bf(C step)}] \textbf{Combine submatrix estimates:} \projmf generates a
final low-rank estimate $\Lprojmf$ by projecting $[\hat\C_1,\ldots,\hat\C_t]$
onto the column space of $\hat \C_1$, \rpmf uses random projection to compute a
rank-$k$ estimate $\Lrpmf$ of $[\hat\C_1 \cdots \hat\C_t]$ where $k$ is the
median rank of the returned subproblem estimates, and \gnysmf forms the
low-rank estimate $\Lgnysmf$ from $\hat \C$ and $\hat \R$ via the generalized
\nys method. These matrix approximation techniques are described in more detail
in Sec.~\ref{sec:rand_mat_approx}.
\ifdefined\jacmformat
\end{describe}
\else
\end{description}
\fi

\begin{figure*}[h]
\renewcommand*\footnoterule{}
\begin{minipage}[t]{.47\textwidth}
\begin{algorithm}[H]
   \caption{\projmf}
   \label{alg:fastmf-proj}
\begin{algorithmic}
	 \STATE {\bfseries Input:} $\obsproj(\M)$, $t$
	 \STATE $\{\obsproj(\C_i)\}_{1\le i \le t}$ = \textsc{SampCol}($\obsproj(\M)$, $t)$ \\
   \textbf{do in parallel}
   \STATE\hspace{5mm} $\hat{\C}_1$ = \textsc{Base-MF-Alg}$(\obsproj(\C_1))$ \\
   \STATE\hspace{30mm} \vdots \\
   \vspace{1mm}
   \STATE\hspace{5mm} $\hat{\C}_t$ = \textsc{Base-MF-Alg}$(\obsproj(\C_t))$\\ 
   \textbf{end do}
	 \STATE $\Lprojmf$ = \textsc{ColProjection}($\hat{\C}_1, \ldots, \hat{\C}_t)$
\end{algorithmic}
\end{algorithm}
\end{minipage}
\hfill
\begin{minipage}[t]{.5\textwidth}
\begin{algorithm}[H]
   \caption{\rpmf}
   \label{alg:fastmf-rp}
\begin{algorithmic}
	 \STATE {\bfseries Input:} $\obsproj(\M)$, $t$
	 \STATE $\{\obsproj(\C_i)\}_{1\le i \le t}$ = \textsc{SampCol}($\obsproj(\M)$, $t)$ \\
   \textbf{do in parallel}
   \STATE\hspace{5mm} $\hat{\C}_1$ = \textsc{Base-MF-Alg}$(\obsproj(\C_1))$ \\
   \STATE\hspace{30mm} \vdots \\
   \vspace{1mm}
   \STATE\hspace{5mm} $\hat{\C}_t$ = \textsc{Base-MF-Alg}$(\obsproj(\C_t))$\\ 
   \textbf{end do}
   \STATE $k = \text{median}_{i\in\{1\ldots t\}}\big(\text{rank} (\hat\C_i)\big)$
	 \STATE $\Lprojmf$ = \textsc{RandProjection}($\hat{\C}_1, \ldots, \hat{\C}_t, k)$
\end{algorithmic}
\end{algorithm}
\end{minipage}
\hfill
\begin{center}
\begin{minipage}[t]{.54\textwidth}
\begin{algorithm}[H]
\caption{\gnysmf}
   \label{alg:fastmf-nys}
\begin{algorithmic}
	 \STATE {\bfseries Input:} $\obsproj(\M)$, $l$, $d$
	 \STATE $\obsproj(\C) \,, \obsproj(\R)$ = \textsc{SampColRow}($\obsproj(\M)$, $l$, $d)$ \\
   \textbf{do in parallel}
   \STATE\hspace{5mm} $\hat{\C}$ = \textsc{Base-MF-Alg}$(\obsproj(\C))$
   \STATE\hspace{5mm} $\hat{\R}$ = \textsc{Base-MF-Alg}$(\obsproj(\R))$\\ 
   \textbf{end do}
	 \STATE $\Lgnysmf$ = \textsc{Gen\nys}($\hat{\C}$, $\hat{\R})$
\end{algorithmic}
\end{algorithm}
\end{minipage}
\end{center}
\end{figure*}

\subsection{Randomized Matrix Approximations}
\label{sec:rand_mat_approx}
Underlying the C step of each DFC algorithm is a method for generating randomized
low-rank approximations to an arbitrary matrix $\M$. %

\paragraph{Column Projection} \projmf (Algorithm~\ref{alg:fastmf-proj}) uses the column projection method
of \citet{FriezeKaVe98}.  
Suppose that $\C$ is a matrix of $l$ columns sampled uniformly and without
replacement from the columns of $\M$.
Then, column projection generates a ``matrix projection'' approximation \citep{KumarMoTa09} of $\M$ via
\begin{equation}
\label{eq:mat_proj_defined}
\Lproj = \C \pinv{\C} \M  = \U_{C}\U_{C}^\top \M.
\end{equation}
In practice, we do not reconstruct $\Lproj$ but rather maintain low-rank
factors, e.g., $\U_{C}$ and $\U_{C}^{\top}\M$.

\paragraph{Random Projection}
The celebrated result of \citet{JohnsonLi84} shows
that random low-dimensional embeddings preserve Euclidean geometry.  Inspired
by this result, several random projection algorithms
\cite[e.g.,][]{PapadimitriouHiRaVe98, Liberty09,RokhlinSzTy09} have been introduced
for approximating a matrix by projecting it onto a random low-dimensional subspace
(see~\citet{HaMaTr11} for further discussion).  
\rpmf (Algorithm~\ref{alg:fastmf-rp}) utilizes such a random projection method due to
\citet{HaMaTr11}.
Given a target low-rank parameter $k$, let $\G$ be an $n \times (k+p)$ standard Gaussian matrix $\G$, where $p$ is an
oversampling parameter. 
Next, let $\Y=(\M\M^\top)^q\M\G$, and define $\Q
\in \reals^{m \times k}$ as the top $k$ left singular vectors of $\Y$.  The
random projection approximation of $\M$ is then given by
\begin{equation}
\label{eq:rp_defined}
\Lrp = \Q \pinv{\Q} \M.
\end{equation}
We work with an implementation \citep{Tygert09} of a numerically stable variant
of this algorithm described in Algorithm $4.4$ of \citet{HaMaTr11}.  Moreover,
the parameters $p$ and $q$ are typically set to small positive constants
\citep{Tygert09, HaMaTr11}, and we set $p=5$ and $q=2$.  

\paragraph{Generalized \nys Method} 
The \nys method was developed for the discretization of integral equations~\citep{Nystrom30} and has
since been used to speed up large-scale learning applications involving 
symmetric positive semidefinite matrices~\citep{WilliamsSe00}. 
\gnysmf (Algorithm~\ref{alg:fastmf-nys}) makes use of a generalization of the \nys method for arbitrary real
matrices~\citep{GoreinovTyZa97}. 
Suppose that $\C$ consists of $l$ columns of $\M$, sampled uniformly without replacement,
and that $\R$ consists of $d$ rows of $\M$, independently sampled uniformly and without replacement.
Let $\W$ be the $d\times l$ matrix formed by sampling the corresponding rows of $\C$.\footnote{This choice is arbitrary: $\W$ could also be defined as a submatrix of $\R$.}
Then, the generalized \nys\ method computes
a ``spectral reconstruction'' approximation \citep{KumarMoTa09} of $\M$ via
\begin{equation}
\label{eq:gnys_defined}
	\Lgnys = \C\pinv{\W}\R = \C\V_{W}\pinv{\mSigma_{W}}\U_{W}^{\top}\R\,.
\end{equation}
As with $\Mproj$, we store low-rank factors of $\Lgnys$, such as
$\C\V_{W}\pinv{\mSigma_{W}}$ and $\U_{W}^{\top}\R$. 

\subsection{Running Time of \fastmf} 
Many state-of-the-art MF algorithms have $\Omega(mnk_M)$ per-iteration time
complexity due to the rank-$k_M$ truncated SVD performed on each iteration.
\fastmf significantly reduces the per-iteration complexity to O$(mlk_{C_i})$
time for $\C_i$ (or $\C$) and O$(ndk_{R})$ time for $\R$. The cost of combining
the submatrix estimates is even smaller when using column projection or the
generalized \nys method, since the outputs of standard MF algorithms are
returned in factored form.  Indeed, if we define $k' \defeq\max_i k_{C_i}$,
then the column projection step of \projmf requires only O$(mk'^2+lk'^2)$ time:
O$(mk'^2+lk'^2)$ time for the pseudoinversion of $\hat\C_1$ and O$(m
k'^2+lk'^2)$ time for matrix multiplication with each $\hat\C_i$ in parallel.
Similarly, the generalized \nys step of \gnysmf requires only O$(l\bar
k^2+d\bar k^2+\minarg{m,n}\bar k^2)$ time, where $\bar
k\defeq\maxarg{k_C,k_R}$.   

\rpmf also benefits from the factored form of the outputs of standard MF
algorithms.  Assuming that $p$ and $q$ are positive constants, the random
projection step of \rpmf requires O($mkt + mkk' + nk$) time where $k$ is the
low-rank parameter of $\Q$: O($nk$) time to generate $\G$, O($mkk'+ mkt$) to
compute $\Y$ in parallel, O($mk^2$) to compute the SVD of $\Y$, and O$(m
k'^2+lk'^2)$ time for matrix multiplication with each $\hat\C_i$ in parallel in
the final projection step. Note that the running time of the random projection step
depends on $t$ (even when executed in parallel) and thus has a larger 
complexity than the column projection and generalized \nys variants.
Nevertheless, the random projection step need be performed only once and thus
yields a significant savings over the repeated computation of SVDs
required by typical base algorithms.

\subsection{Ensemble Methods}
Ensemble methods  have been shown to improve performance of matrix
approximation algorithms, while straightforwardly leveraging the parallelism of modern
many-core and distributed architectures \citep{KumarMoTa09b}.  As such, we
propose ensemble variants of the \fastmf algorithms that demonstrably reduce
estimation error while introducing a negligible cost to the parallel running time.  
For \projmfens, rather than projecting only onto the column space of
$\hat\C_1$, we project $[\hat\C_1,\ldots,\hat\C_t]$ onto the column space of
each $\hat\C_i$ in parallel and then average the $t$ resulting low-rank
approximations. For \rpmfens, rather than projecting only onto a column space
derived from a single random matrix $\G$, we project
$[\hat\C_1,\ldots,\hat\C_t]$ onto $t$ column spaces derived from $t$ random
matrices in parallel and then average the $t$ resulting low-rank
approximations.  For \gnysmfens, we choose a random $d$-row submatrix
$\obsproj(\R)$ as in \gnysmf and independently partition the columns of
$\obsproj(\M)$ into $\{\obsproj(\C_1),\ldots,\obsproj(\C_t)\}$ as in \projmf
and \rpmf.  After running the base MF algorithm on each submatrix, we apply the
generalized \nys method to each $(\hat\C_i,\hat\R)$ pair in parallel and
average the $t$ resulting low-rank approximations.  Sec.~\ref{sec:experiments}
highlights the empirical effectiveness of ensembling.

\section{Roadmap of Theoretical Analysis} %
\label{sec:roadmap}
While \fastmf in principle can work with any base matrix factorization algorithm,
it offers the greatest benefits when united with accurate but computationally expensive
base procedures.
Convex optimization approaches to matrix completion and robust matrix factorization 
\citep[e.g., ][]{LinGaWrWuChMa09,MaGoCh11,TohYu10} are prime examples
of this class, since they admit strong theoretical estimation guarantees~\citep{AgarwalNeWa11,CandesLiMaWr09,CandesPl10,NegahbanWa12}
but suffer from poor computational complexity due to the repeated and 
costly computation of truncated SVDs.  
Sec.~\ref{sec:experiments} will provide empirical evidence that \fastmf 
provides an attractive framework to improve the scalability of these algorithms, 
but we first present a thorough theoretical analysis of the estimation properties of \fastmf.

Over the course of the next three sections, we will show that the same assumptions
that give rise to strong estimation guarantees for standard MF formulations also 
guarantee strong estimation properties for \fastmf.
In the remainder of this section, we first introduce these standard assumptions 
and then present simplified bounds to build intuition for our
theoretical results and our underlying proof techniques. 

\subsection{Standard Assumptions for Noisy Matrix Factorization}

Since not all matrices can be recovered from missing entries or gross outliers,
recent theoretical advances have studied sufficient conditions for accurate
noisy MC~\citep{CandesPl10,KeshavanMoOh10b,NegahbanWa12} and
RMF~\citep{AgarwalNeWa11,ZhouLiWrCaMa10}. Informally, these conditions
capture the degree to which information about a single entry is ``spread out'' across a matrix. 
The ease of matrix estimation is correlated with 
this spread of information. 
The most prevalent set of conditions are \emph{matrix coherence}
conditions, which limit the extent to which the singular vectors of a matrix
are correlated with the standard basis.  However, there exist classes of
matrices that violate the coherence conditions but can nonetheless be recovered
from missing entries or gross outliers. \citet{NegahbanWa12} define an
alternative notion of \emph{matrix spikiness} in part to handle these classes. 

\subsubsection{Matrix Coherence}
\label{sec:coherence}

Letting $\evec_i$ be the $i$th column of the standard basis, we define two
standard notions of coherence \citep{Recht11}:
\begin{definition}[$\mu_0$-Coherence]
\label{defn:coherence_0}
Let $\V \in \reals^{n \times r}$ contain orthonormal columns with $r \leq n$.
Then the $\mu_0$-coherence of $\V$ is:
\begin{equation*}
\textstyle \mu_0(\V) \defeq \frac{n}{r} \max_{1 \le i \le n} \norm{\P_V \evec_i}^2 =
\frac{n}{r} \max_{1 \le i \le n} \norm{\V_{(i)}}^2  \,. 
\end{equation*} 
\end{definition}

\begin{definition}[$\mu_1$-Coherence]
\label{defn:coherence_1}
Let $\L \in \reals^{m \times n}$ have rank $r$.  Then, the $\mu_1$-coherence of
$\L$ is: 
\begin{equation*}
\textstyle \mu_1(\L) \defeq \sqrt{\frac{mn}{r}} \max_{ij} |\evec_i^\top\U_L\V_L^\top\evec_j| \,. 
\end{equation*} 
\end{definition}
For conciseness, we extend the definition of $\mu_0$-coherence to an
arbitrary matrix $\L \in \reals^{m \times n}$ with rank $r$ via
$\mu_0(\L) \defeq \maxarg{\mu_0(\U_L),\mu_0(\V_L)}.$
Further, for any $\mu>0$, we will call a matrix $\L$ \emph{$(\mu,r)$-coherent} if
$\rank{\L}=r$, $\mu_0(\L) \leq\mu$, and
$\mu_1(\L)\leq\sqrt{\mu}$.  Our analysis in Sec.~\ref{sec:theory} will focus on base MC and RMF
algorithms that express their estimation guarantees in terms of the
$(\mu,r)$-coherence of the target low-rank matrix $\L_0$.  For such algorithms,
lower values of $\mu$ correspond to better estimation properties.

\subsubsection{Matrix Spikiness}
\label{sec:spikiness}
The matrix spikiness condition of \citet{NegahbanWa12} captures the intuition
that a matrix is easier to estimate if its maximum entry is not much larger than
its average entry (in the root mean square sense):
\begin{definition}[Spikiness]
\label{defn:coherence_0}
The spikiness of $\L \in \reals^{m \times n}$ is: 
\begin{equation*}
\textstyle \spikiness(\L) \defeq \sqrt{mn}\norm{\L}_\infty/\norm{\L}_F.
\end{equation*} 
We call a matrix $\spikiness$-spiky if $\spikiness(\L) \leq \spikiness$.
\end{definition}
Our analysis in Sec.~\ref{sec:theory-spike} will focus on base MC algorithms
that express their estimation guarantees in terms of the $\spikiness$-spikiness
of the target low-rank matrix $\L_0$.  For such algorithms, lower values of
$\spikiness$ correspond to better estimation properties.
\subsection{Prototypical Estimation Bounds}

We now present a prototypical estimation bound for \fastmf. 
Suppose that a base MC algorithm solves the \emph{noisy nuclear norm heuristic}, studied in \citet{CandesPl10}:
\begin{align*}
\text{minimize}_{\L} \quad \norm{\L}_* \quad \text{subject\, to}\quad
\norm{\obsproj(\M-\L)}_F \leq \Delta,
\end{align*}
and that, for simplicity, $\M$ is square.
The following prototype bound, derived from a new noisy MC guarantee in Thm.~\ref{thm:convex-mc-noise}, 
describes the behavior of this estimator under matrix coherence assumptions.
Note that the bound implies exact recovery in the noiseless setting, i.e., when $\Delta = 0$.

\begin{proto}[MC under Incoherence]
\label{ex:proto-mc}
Suppose that $\L_0$ is $(\mu,r)$-coherent, $s$
entries of $\M \in \reals^{n \times n}$ are observed uniformly at random where $s = \Omega(\mu r n \log^2(n))$, and
$\norm{\M - \L_0}_F \le \Delta$. If $\hat \L$ solves the noisy nuclear norm
heuristic, then
$$\norm{\L_0 - \hat \L}_F \leq f(n) \Delta$$
with high probability, where $f$ is a function of $n$.
\end{proto}

Now we present a corresponding prototype bound for \projmf, a simplified version of 
our Cor.~\ref{cor:fast-mc-noise}, under precisely the same coherence assumptions.
Notably, this bound i) preserves accuracy with a flexible $(2+\epsilon)$ degradation in estimation
error over the base algorithm, ii) allows for speed-up by requiring only a vanishingly small fraction of columns to be sampled (i.e., $l/n \rightarrow 0$) whenever $s = \omega(n\log^2(n))$ entries are revealed, and iii) maintains exact recovery in the
noiseless setting. %
\begin{proto}[\fastmf-MC under Incoherence]
\label{ex:fast-mc-noise}
Suppose that $\L_0$ is $(\mu,r)$-coherent, $s$ entries of $\M \in \reals^{n \times n}$ are
observed uniformly at random, and $\norm{\M - \L_0}_F \le \Delta$. Then
\begin{align*}
	l = O\bigg(\frac{\mu^2 r^2n^2\log^2(n)}{s\epsilon^2}\bigg) 
\end{align*}
random columns suffice to have
$$\norm{\L_0 - \Lprojmf}_F \leq (2+\epsilon)f(n) \Delta$$
with high probability when the noisy nuclear norm heuristic is used as a base
algorithm, where $f$ is the same function of $n$ defined in Proto.~\ref{ex:proto-mc}.
\end{proto}
The proof of Proto.~\ref{ex:fast-mc-noise}, and indeed of each of our main \fastmf results, consists of three high-level steps:
\begin{enumerate}
\item{\emph{Bound information spread of submatrices}}: 
Recall that the F step of \fastmf operates by applying a base MF algorithm to submatrices.
We show that,  with high probability, uniformly sampled submatrices are only moderately more coherent and moderately more spiky than the matrix from which they are drawn.
This allows for accurate estimation of submatrices using base algorithms with standard coherence or spikiness requirements.
The conservation of incoherence result is summarized in Lem.~\ref{lem:sub-coh}, while the conservation of non-spikiness is presented in Lem.~\ref{lem:sub-spike}.
\item{\emph{Bound error of randomized matrix approximations}}: 
The error introduced by the C step of \fastmf depends on the framework variant.
Drawing upon tools from randomized $\ell_2$ regression~\citep{DrineasMaMu08}, randomized matrix multiplication~\citep{DrineasKaMa06a,DrineasKaMa06b}, and matrix concentration~\citep{HsuKaZh12},
we show that the same assumptions on the spread of information responsible for accurate MC and RMF also yield high fidelity reconstructions for column projection (Cor.~\ref{cor:proj-main} and Thm.~\ref{cor:proj-main-spike}) and the \nys method (Cor.~\ref{cor:gnys-main} and Cor.~\ref{cor:gnys-low-rank}).
We additionally present general approximation guarantees for random projection due to \citet{HaMaTr11}
in Cor.~\ref{cor:rp-main}.    
These results give rise to ``master theorems'' for coherence (Thm.~\ref{thm:master})
and spikiness (Thm.~\ref{thm:master-spike}) that generically relate the estimation error of \fastmf to the error of any base algorithm.
\item{\emph{Bound error of submatrix factorizations}}: 
The final step combines a master theorem with a base estimation guarantee applied to each
\fastmf subproblem.
We study both new (Thm.~\ref{thm:convex-mc-noise}) and established bounds (Thm.~\ref{thm:rpca-noise} and Cor.~\ref{cor:mc-spike})
 for MC and RMF and prove that
\fastmf submatrices satisfy the base guarantee preconditions with high probability.
We present the resulting coherence-based estimation guarantees for DFC in 
Cor.~\ref{cor:fast-mc-noise} and Cor.~\ref{cor:fast-rpca-noise} and the spikiness-based
estimation guarantee in Cor.~\ref{cor:fast-mc-spike}.
\end{enumerate}

The next two sections present the main results contributing to each of these proof steps,
as well as their consequences for MC and RMF.
Sec.~\ref{sec:theory} presents our analysis under coherence assumptions,
while Sec.~\ref{sec:theory-spike} contains our spikiness analysis.

\section{Coherence-based Theoretical Analysis}
\label{sec:theory}
\subsection{Coherence Analysis of Randomized Approximation Algorithms}
\label{sec:rand_mat_analysis}
We begin our coherence-based analysis by characterizing the behavior of randomized approximation
algorithms under standard coherence assumptions. 
The derived properties will aid us in deriving \fastmf estimation guarantees.  
Hereafter,  $\epsilon\in (0,1]$ represents a prescribed error
tolerance, and $\delta,\delta'\in(0,1]$ denote target failure probabilities.

\subsubsection{Conservation of Incoherence}
Our first result bounds the $\mu_0$ and $\mu_1$-coherence of a uniformly sampled
submatrix in terms of the coherence of the full matrix.
This conservation of incoherence allows for accurate submatrix completion or
submatrix outlier removal when using standard MC and RMF algorithms.  
Its proof is given in Sec.~\ref{sec:sub-coh}.

\begin{lemma}[Conservation of Incoherence]
\label{lem:sub-coh}
Let $\L\in\reals^{m\times n}$ be a rank-$r$ matrix and define $\L_C\in\reals^{m\times
l}$ as a matrix of $l$ columns of $\L$ sampled uniformly without replacement.
If $l \geq cr\mu_0(\V_{L})\log(n)\log(1/\delta)/\epsilon^2,$ where $c$ is a fixed
positive constant defined in Cor.~\ref{cor:proj-main}, then 
\begin{enumerate}
\item[i)]$\rank{\L_C} = \rank{\L}$
\item[ii)]$\mu_0(\U_{L_C}) = \mu_0(\U_L)$
\item[iii)]$\displaystyle\mu_0(\V_{L_C}) \le \frac{\mu_0(\V_L)}{1-\epsilon/2}$
\item[iv)]$\displaystyle\mu_1^2(\L_C) \le \frac{r\mu_0(\U_L)\mu_0(\V_L)}{1-\epsilon/2}$
\end{enumerate}
all hold jointly with probability at least $1-\delta/n$.
\end{lemma}

\subsubsection{Column Projection Analysis}
Our next result shows that projection based on uniform column sampling 
leads to near optimal estimation in matrix regression when the covariate matrix has small coherence.
This statement will immediately give rise to estimation guarantees for column projection and the generalized \nys method. 
\begin{theorem}[Subsampled Regression under Incoherence]
\label{thm:regress-main}
Given a target matrix $\B\in\reals^{p\times n}$ and a rank-$r$ matrix of covariates
$\L\in\reals^{m\times n}$, choose $l\geq 3200r\mu_0(\V_{L})\log(4n/\delta)/\epsilon^2,$
let $\B_C\in\reals^{p\times l}$ be a matrix of $l$ columns of $\B$
sampled uniformly without replacement,
and let $\L_C\in\reals^{m\times l}$ consist of the corresponding columns of $\L$. 
Then, $$\norm{\B - \B_C\pinv{\L_C}\L}_F \leq
(1+\epsilon)\norm{\B - \B\pinv{\L}\L}_F$$ with probability at least $1-\delta-0.2$.
\end{theorem}

Fundamentally, Thm.~\ref{thm:regress-main} links the notion of coherence, common in
matrix estimation communities, to the randomized approximation concept of
\emph{leverage score sampling}~ \citep{MahoneyDr09}.
The proof of Thm.~\ref{thm:regress-main}, given in Sec.~\ref{sec:regress-proof}, builds upon the randomized $\ell_2$ regression
work of \citet{DrineasMaMu08} and the matrix concentration results of
\citet{HsuKaZh12} to yield a subsampled regression guarantee with better sampling complexity
than that of \citet[Thm.~5]{DrineasMaMu08}.

A first consequence of Thm.~\ref{thm:regress-main} shows that, with high probability, 
column projection produces an estimate nearly as good as a given rank-$r$ target by 
sampling a number of columns proportional to the coherence and $r\log n$. 
\begin{corollary}[Column Projection under Incoherence]
\label{cor:proj-main}
Given a matrix $\M\in\reals^{m\times n}$ and a rank-$r$ approximation
$\L\in\reals^{m\times n}$, choose $l \geq
cr\mu_0(\V_{L})\log(n)\log(1/\delta)/\epsilon^2,$ where $c$ is a fixed positive
constant, 
and let $\C\in\reals^{m\times l}$ be a matrix of $l$ columns of $\M$
sampled uniformly without replacement.  Then, $$\norm{\M - \C\pinv{\C}\M}_F \leq
(1+\epsilon)\norm{\M - \L}_F$$ with probability at least $1-\delta$.
\end{corollary}
Our result generalizes Thm.~1 of \citet{DrineasMaMu08} by providing 
improved sampling complexity and guarantees
relative to an \emph{arbitrary} low-rank approximation.
Notably, in the ``noiseless'' setting, when $\M = \L$, Cor.~\ref{cor:proj-main} 
guarantees exact recovery of $\M$ with high probability.
The proof of Cor.~\ref{cor:proj-main} is given in Sec.~\ref{sec:proj-proof}.

\subsubsection{Generalized \nys Analysis}
Thm.~\ref{thm:regress-main} and Cor.~\ref{cor:proj-main} together imply an
estimation guarantee for the generalized \nys method 
relative to an arbitrary low-rank approximation $\L$.
Indeed, if the matrix of sampled columns is denoted by $\C$, then,
with appropriately reduced probability, O($\mu_0(\V_L)r\log n$) columns and O($\mu_0(\U_C)r\log m$) rows
suffice to match the reconstruction error of $\L$ up to any fixed precision.
The proof can be found in Sec.~\ref{sec:gnys-proof}.
\begin{corollary}[Generalized \nys under Incoherence]
\label{cor:gnys-main}
Given a matrix $\M\in\reals^{m\times n}$ and a rank-$r$ approximation
$\L\in\reals^{m\times n}$, choose 
$l \geq cr\mu_0(\V_{L})\log(n)\log(1/\delta)/\epsilon^2$ with $c$ a constant as in Cor.~\ref{cor:proj-main}, 
and let $\C\in\reals^{m\times l}$ be a matrix of $l$ columns of $\M$
sampled uniformly without replacement.  
Further choose
$d \geq cl\mu_0(\U_{C})\log(m)\log(1/\delta')/\epsilon^2,$ 
and let $\R\in\reals^{d\times n}$ be a matrix of $d$ rows of $\M$
sampled independently and uniformly without replacement.  
Then, $$\norm{\M - \C\pinv{\W}\R}_F \leq
(1+\epsilon)^2\norm{\M - \L}_F$$ with probability at least $(1-\delta)(1-\delta'-0.2)$.
\end{corollary}

Like the generalized \nys bound of \citet[Thm.~4]{DrineasMaMu08} and unlike our column projection result, 
Cor.~\ref{cor:gnys-main} depends on the coherence of the submatrix $\C$
and holds only with probability bounded away from 1.
Our next contribution shows that we can do away with these restrictions in the noiseless setting, where $\M = \L$.
\begin{corollary}[Noiseless Generalized \nys under Incoherence]
  \label{cor:gnys-low-rank}
Let $\L \in \reals^{m\times n}$ be a rank-$r$ matrix. 
Choose $l \geq 48r \mu_0(\V_L)\log(4n/(1-\sqrt{1-\delta}))$ and $d \geq 48r
\mu_0(\U_L)\log(4m/(1-\sqrt{1-\delta}))$.  Let $\C\in\reals^{m\times l}$ be a
matrix of $l$ columns of $\L$ sampled uniformly without replacement, and   let
$\R\in\reals^{d\times n}$ be a matrix of $d$ rows of $\L$ sampled independently
and uniformly without replacement.  Then, 
$$\L = \C\pinv{\W}\R$$ 
with probability at least $1-\delta$.
\end{corollary}
The proof of Cor.~\ref{cor:gnys-low-rank}, given in Sec.~\ref{sec:gnys-low-rank-proof}, adapts a strategy of \citet{TalwalkarRo10} developed for the analysis of positive semidefinite matrices.

\subsubsection{Random Projection Analysis}
We next present an estimation guarantee for the random projection method
relative to an arbitrary low-rank approximation $\L$.  
The result implies that using a random matrix with oversampled
columns proportional to $r\log(1/\delta)$ suffices to match the reconstruction
error of $\L$ up to any fixed precision with probability $1-\delta$.
The result is a direct consequence of the random projection analysis of \citet[Thm.~10.7]{HaMaTr11}, and the proof can be found in Sec.~\ref{sec:rp-proof}.

\begin{corollary}[Random Projection]
\label{cor:rp-main}
Given a matrix $\M\in\reals^{m\times n}$ and a rank-$r$ approximation
$\L\in\reals^{m\times n}$ with $r \ge 2$, choose an oversampling parameter
$$p \ge 242\ r\log(7/\delta)/\epsilon^2.$$
Draw an $n \times (r+p)$ standard Gaussian
matrix $\G$ and define $\Y = \M \G$.  
Then, with probability at least
$1-\delta$, $$\norm{\M - \P_{Y} \M}_F \leq (1+\epsilon) \norm{\M - \L}_F.$$
Moreover, define $\Lrp$ as the best rank-$r$ approximation of $\P_{Y} \M$ with
respect to the Frobenius norm. Then, with probability at least $1-\delta$,
$$\norm{\M - \Lrp}_F \leq (2+\epsilon) \norm{\M - \L}_F.$$
\end{corollary}

We note that, in contrast to
Cor.~\ref{cor:proj-main} and Cor.~\ref{cor:gnys-main},
Cor.~\ref{cor:rp-main} does not depend on the coherence of $\L$ and hence
can be fruitfully applied even in the absence of an incoherence assumption.
We demonstrate such a use case in Sec.~\ref{sec:theory-spike}.

\subsection{Base Algorithm Guarantees} 
As prototypical examples of the coherence-based estimation guarantees available for
noisy MC and noisy RMF, consider the following two theorems.
The first bounds the estimation error of a convex optimization approach to noisy matrix completion, 
under the assumptions of incoherence and uniform sampling.
\begin{theorem}[Noisy MC under Incoherence]
\label{thm:convex-mc-noise}
Suppose that $\L_0\in\reals^{m\times n}$ is $(\mu,r)$-coherent and that, for
some target rate parameter $\beta > 1$, $$s \geq 32 \mu
r(m+n)\beta\log^2(m+n)$$ entries of $\M$ are observed with locations $\obsset$
sampled uniformly without replacement.  Then, if $m\leq n$ and
$\norm{\obsproj(\M)-\obsproj(\L_0)}_F \leq \Delta$ a.s., the minimizer
$\hat{\L}$ of the problem 
\begin{align}
\label{eqn:convex-mc-noise}
\text{minimize}_{\L} \quad \norm{\L}_* \quad \text{subject\, to}\quad
\norm{\obsproj(\M-\L)}_F \leq \Delta.
\end{align}
satisfies $$\norm{\L_0 - \hat{\L}}_F \leq 8\sqrt{\frac{2m^2n}{s}+m+\frac{1}{16}
}\Delta \leq c_e\sqrt{mn}\Delta$$ with probability at least $1 - 4
\log(n)n^{2-2\beta}$ for $c_e$ a positive constant.
\end{theorem}
A similar estimation guarantee was obtained by \citet{CandesPl10} under stronger assumptions.
We give the proof of Thm.~\ref{thm:convex-mc-noise} in Sec.~\ref{sec:convex-mc-noise}.

The second result, due to~\citet{ZhouLiWrCaMa10} and reformulated for a generic rate parameter $\beta$,
as described in~\citet[Section~3.1]{CandesLiMaWr09}, 
bounds the estimation error of a convex optimization approach to noisy RMF,
under the assumptions of incoherence and uniformly distributed outliers.
\begin{theorem}[Noisy RMF under Incoherence~{\cite[Thm.~2]{ZhouLiWrCaMa10}}]
\label{thm:rpca-noise}
Suppose that $\L_0$ is $(\mu,r)$-coherent and that the support set of $\S_0$ is
uniformly distributed among all sets of cardinality $s$.  Then, if $m\leq n$
and $\norm{\M-\L_0-\S_0}_F\leq \Delta$ a.s., there is a constant $c_p$ such
that with probability at least $1-c_pn^{-\beta}$, the minimizer
$(\hat{\L},\hat{\S})$ of the problem
\begin{align}
\label{eq:pcp-noise}
&\text{minimize}_{\L,\S} \quad \norm{\L}_*+\lambda\norm{\S}_1\quad
\text{subject\, to}\quad \norm{\M - \L - \S}_F \leq \Delta\quad
\text{with}\quad \lambda=1/\sqrt{n}
\end{align}
satisfies $\norm{\L_0 - \hat{\L}}_F^2 + \norm{\S_0 -
\hat{\S}}_F^2 \leq c_e'^2mn\Delta^2$, provided that $$r \leq \frac{\rho_r
m}{\mu\log^2(n)}\quad and\quad s\leq(1-\rho_s\beta) mn$$ for target rate
parameter $\beta > 2$, and positive constants $\rho_r, \rho_s,$ and $c_e'$.
\end{theorem}

\subsection{Coherence Master Theorem} 
We now show that the same coherence conditions that allow for accurate MC and
RMF also imply high-probability estimation guarantees for \fastmf.  To make this precise, we
let $\M = \L_0 + \S_0 + \Z_0 \in \reals^{m \times n}$, where  $\L_0$ is
$(\mu,r)$-coherent and $\norm{\obsproj(\Z_0)}_F \le \Delta$.  
Then, our next theorem provides a generic bound on the estimation error of \fastmf used
in combination with an arbitrary base algorithm.  
The proof, which builds upon the results of Sec.~\ref{sec:rand_mat_analysis}, is given in
Sec.~\ref{sec:proof-master}.

\begin{theorem}[Coherence Master Theorem]
\label{thm:master}
Choose $t=n/l$, $l\geq cr\mu\log(n)\log(2/\delta)/\epsilon^2$, where $c$ is a
fixed positive constant, and $p \ge 242\ r\log(14/\delta)/\epsilon^2$. 
Under the notation of Algorithms~$\ref{alg:fastmf-proj}$ and~$\ref{alg:fastmf-rp}$,
let $\{ \C_{0,1}, \cdots, \C_{0,t} \}$ be the corresponding partition
of $\L_0$. Then, with probability at least $1 - \delta$, 
$\C_{0,i}$ is $(\frac{r\mu^2}{1-\epsilon/2},r)$-coherent for all $i$, and
\begin{align*}
&\norm{\L_0 - \hat{\L}^*}_F \leq (2+\epsilon)\sqrt{\textsum_{i=1}^t \norm{\C_{0,i} - \hat{\C}_i}_F^2},
\end{align*} 
where $\hat{\L}^*$ is the estimate returned by either \projmf or \rpmf.

Under the notation of Algorithm $\ref{alg:fastmf-nys}$, 
let $\C_0$ and $\R_0$ be the corresponding column and row submatrices of $\L_0$.
If in addition  
$d\geq cl\mu_0(\hat\C)\log(m)\log(4/\delta)/\epsilon^2$, 
then, with probability at least $(1 - \delta)(1-\delta-0.2)$, \gnysmf guarantees that 
$\C_{0}$ and $\R_{0}$ are $(\frac{r\mu^2}{1-\epsilon/2},r)$-coherent and that
$$\norm{\L_0 - \Lgnysmf}_F \leq (2+3\epsilon)\sqrt{\norm{\C_0-\hat\C}_F^2+\norm{\R_0-\hat\R}_F^2}.$$
\end{theorem}
\begin{remark}
	The \gnysmf guarantee requires the number of rows sampled to grow in proportion to $\mu_0(\hat\C)$, a
quantity always bounded by $\mu$ in our simulations.
	Here and in the consequences to follow, the \gnysmf result can be strengthened in the noiseless setting ($\Delta = 0$) by utilizing 
Cor.~\ref{cor:gnys-low-rank} in place of Cor.~\ref{cor:gnys-main} in the proof of Thm.~\ref{thm:master}.
\end{remark}

When a target matrix is incoherent, Thm.~\ref{thm:master} asserts that 
-- with high probability for \projmf and \rpmf and with fixed probability for \gnysmf\ --
the estimation error of \fastmf is not much larger than the error sustained by the base algorithm on each subproblem.
Because Thm.~\ref{thm:master} further bounds the coherence of each submatrix, we can use any
coherence-based matrix estimation guarantee to control the estimation error on each subproblem.
The next two sections demonstrate how Thm.~\ref{thm:master} can be applied to
derive specific \fastmf estimation guarantees for noisy MC and noisy RMF.
In these sections, we let $\bar{n}\defeq\maxarg{m,n}$.

\subsection{Consequences for Noisy MC}
\label{sec:cons_mc}
As a first consequence of Thm.~\ref{thm:master}, we will show that \fastmf retains the
high-probability estimation guarantees of a standard MC solver while operating on
matrices of much smaller dimension.  
Suppose that a base MC algorithm solves the convex optimization problem of \eqref{eqn:convex-mc-noise}.
Then, Cor.~\ref{cor:fast-mc-noise} follows from the Coherence Master Theorem (Thm.~\ref{thm:master}) and the base algorithm guarantee of 
Thm.~\ref{thm:convex-mc-noise}.
\begin{corollary}[\fastmf-MC under Incoherence]
\label{cor:fast-mc-noise}
Suppose that $\L_0$ is $(\mu,r)$-coherent and that $s$ entries of $\M$ are
observed, with locations $\obsset$ distributed uniformly.  
Fix any target rate parameter $\beta > 1$.  
Then, if $\norm{\obsproj(\M)-\obsproj(\L_0)}_F \leq \Delta$ a.s., 
and the base algorithm solves the optimization problem of \eqref{eqn:convex-mc-noise}, 
it suffices to choose $t=n/l,$
\begin{gather*}
	l \geq
	\textstyle{c\mu^2 r^2(m+n)n\beta\log^2(m+n)}/(s\epsilon^2),\quad 
	d \geq
	\textstyle{cl\mu_0(\hat{\C})(2\beta-1)\log^2(4\bar{n})}\bar{n}/(n\epsilon^2),
\end{gather*}
and $p \geq 242\ r\log(14\bar{n}^{2\beta-2})/\epsilon^2$ to achieve 
\vspace{-2mm}
\ifdefined\jacmformat
\begin{describe}{\projmf$:$}
\setlength\itemindent{25pt}
\else
\begin{description}
\setlength\itemindent{25pt}
\fi
\item[\projmf$:$]$\norm{\L_0 - \Lprojmf}_F \leq (2+\epsilon)c_e\sqrt{mn}\Delta$
\item[\rpmf$:$]$\norm{\L_0 - \Lrpmf}_F \leq (2+\epsilon)c_e\sqrt{mn}\Delta$
\item[\gnysmf$:$]$\norm{\L_0 - \Lgnysmf}_F \leq (2+3\epsilon)c_e\sqrt{ml+dn}\Delta $
\ifdefined\jacmformat
\end{describe}
\else
\end{description}
\fi
with probability at least
\ifdefined\jacmformat
\begin{describe}{\projmf\ / \rpmf$:$}
\setlength\itemindent{25pt}
\else
\begin{description}
\fi
\label{eq:fastmf-prob}
\setlength\itemindent{25pt}
\item[\projmf\ / \rpmf$:$]$1 -  (5t\log(\bar{n})+1)\bar{n}^{2-2\beta} \geq
1 -  \bar{n}^{3-2\beta}$
\item[\gnysmf$:$]$1 -  (10\log(\bar{n})+2)\bar{n}^{2-2\beta} - 0.2$,
\ifdefined\jacmformat
\end{describe}
\else
\end{description} 
\fi
respectively, with $c$ as in Thm.~\ref{thm:master} and $c_e$ as in Thm.~\ref{thm:convex-mc-noise}.
\end{corollary}
\begin{remark}
Cor.~\ref{cor:fast-mc-noise} allows for the fraction of columns and
rows sampled to decrease as the number of revealed entries, $s$, increases.
Only a vanishingly small fraction of columns ($l/n\to 0$) and rows ($d/\bar{n} \to 0$)
need be sampled whenever 
$s=\omega((m+n)\log^2(m+n))$.
\end{remark}

To understand the conclusions of Cor.~\ref{cor:fast-mc-noise}, consider the base 
algorithm of Thm.~\ref{thm:convex-mc-noise}, which, when applied to $\obsproj(\M)$, recovers an estimate
$\hat{\L}$ satisfying $\norm{\L_0 - \hat{\L}}_F \leq c_e\sqrt{mn}\Delta$ with
high probability.  
Cor.~\ref{cor:fast-mc-noise} asserts that, with appropriately reduced probability, 
\projmf and \rpmf exhibit the same estimation error scaled by an
adjustable factor of $2+\epsilon$, while \gnysmf exhibits a somewhat smaller
error scaled by $2+3\epsilon$.

The key take-away is that \fastmf introduces a controlled increase in error and a controlled
decrement in the probability of success, allowing the user to interpolate
between maximum speed and maximum accuracy.  Thus, \fastmf can quickly provide
near-optimal estimation in the noisy setting and exact recovery in the noiseless
setting ($\Delta=0)$, even when entries are missing.  
The proof of  Cor.~\ref{cor:fast-mc-noise} can be found in Sec.~\ref{sec:fast-mc-noise}.

\subsection{Consequences for Noisy RMF}
\label{sec:cons_rmf}
Our next corollary shows that \fastmf retains the high-probability estimation
guarantees of a standard RMF solver while operating on matrices of much smaller
dimension.  Suppose that a base RMF algorithm solves the convex
optimization problem of~\eqref{eq:pcp-noise}.  
Then, Cor.~\ref{cor:fast-rpca-noise} follows
from the Coherence Master Theorem (Thm.~\ref{thm:master}) 
and the base algorithm guarantee of Thm.~\ref{thm:rpca-noise}.
\begin{corollary}[\fastmf-RMF under Incoherence]
\label{cor:fast-rpca-noise}
Suppose that $\L_0$ is $(\mu,r)$-coherent 
with $$r^2 \leq \frac{\minarg{m,n} \rho_r}{ 2\mu^2\log^2(\bar{n})}$$ for a positive
constant $\rho_r$.
Suppose moreover that the uniformly distributed
support set of $\S_0$ has cardinality $s$.  For a fixed positive constant
$\rho_s$, define the undersampling parameter $$\beta_s \defeq
\left(1-\frac{s}{mn}\right)/\rho_s,$$ and fix any target rate parameter $\beta
> 2$ with rescaling $\beta' \defeq \beta\log(\bar n)/\log(m)$ satisfying
$4\beta_s-3/\rho_s \leq \beta' \leq \beta_s$.  Then, if
$\norm{\M-\L_0-\S_0}_F\leq \Delta$ a.s.,
and the base algorithm solves the optimization problem of \eqref{eq:pcp-noise}, 
it suffices to choose $t=n/l$,
\begin{align*}
l &\geq \maxarg{\frac{cr^2\mu^2\beta\log^2(2\bar{n})}{\epsilon^2\rho_r},
\frac{4\log(\bar{n})\beta(1-\rho_s\beta_s)}{m(\rho_s\beta_s-\rho_s\beta')^2}}, \\
d &\geq \maxarg{\frac{cl\mu_0(\hat\C)\beta\log^2(4\bar{n})}{\epsilon^2},
\frac{4\log(\bar{n})\beta(1-\rho_s\beta_s)}{n(\rho_s\beta_s-\rho_s\beta')^2}} 
\end{align*}	
and $p \geq 242\ r\log(14\bar{n}^\beta)/\epsilon^2$ to have	
\ifdefined\jacmformat
\begin{describe}{\projmf$:$}
\else
\begin{description}
\fi
\setlength\itemindent{25pt}
\item[\projmf$:$]$\norm{\L_0 - \Lprojmf}_F \leq (2+\epsilon)c_e'\sqrt{mn}\Delta$
\item[\rpmf$:$]$\norm{\L_0 - \Lrpmf}_F \leq (2+\epsilon)c_e'\sqrt{mn}\Delta$
\item[\gnysmf$:$]$\norm{\L_0 - \Lgnysmf}_F \leq (2+3\epsilon)c_e'\sqrt{ml+dn}\Delta $
\ifdefined\jacmformat
\end{describe}
\else
\end{description}
\fi
with probability at least
\ifdefined\jacmformat
\begin{describe}{\projmf\ / \rpmf$:$}
\else
\begin{description}
\fi
\label{eq:fastrmf-prob}
\setlength\itemindent{25pt}
\item[\projmf\ / \rpmf$:$]$1-(t(c_p+1)+1)\bar{n}^{-\beta} \geq
1-c_p\bar{n}^{1-\beta}$
\item[\gnysmf$:$]$1 - (2c_p+3)\bar{n}^{-\beta} - 0.2$,
\ifdefined\jacmformat
\end{describe}
\else
\end{description}
\fi
respectively, with $c$ as in Thm.~\ref{thm:master} and $\rho_r, c_e',$ and
$c_p$ as in Thm.~\ref{thm:rpca-noise}.
\end{corollary}
Note that Cor.~\ref{cor:fast-rpca-noise} places only very mild restrictions on
the number of columns and rows to be sampled.  Indeed, $l$ and $d$ need only
grow poly-logarithmically in the matrix dimensions to achieve estimation guarantees
comparable to those of the RMF base algorithm (Thm.~\ref{thm:rpca-noise}).
Hence, \fastmf can quickly provide near-optimal estimation in the noisy setting and 
exact recovery in the noiseless setting ($\Delta=0)$, even when entries are grossly corrupted.
The proof of  Cor.~\ref{cor:fast-rpca-noise} can be found in Sec.~\ref{sec:fast-rpca-noise}.

\section{Theoretical Analysis under Spikiness Conditions}
\label{sec:theory-spike}
\subsection{Spikiness Analysis of Randomized Approximation Algorithms}
\label{sec:rand_mat_analysis_spike}
We begin our spikiness analysis by characterizing the behavior of randomized 
approximation algorithms under standard spikiness assumptions.  
The derived properties will aid us in developing \fastmf estimation guarantees.
Hereafter,  $\epsilon\in (0,1]$ represents a prescribed error
tolerance, and $\delta,\delta'\in(0,1]$ designates a target failure probability.

\subsubsection{Conservation of Non-Spikiness}

Our first lemma establishes that the uniformly sampled submatrices of an
$\spikiness$-spiky matrix are themselves nearly $\spikiness$-spiky with high
probability.  This property will allow for accurate submatrix completion or
outlier removal using standard MC and RMF algorithms.  Its proof is given in
Sec.~\ref{sec:sub-spike}.

\begin{lemma}[Conservation of Non-Spikiness]
\label{lem:sub-spike}
Let $\L_C\in\reals^{m\times l}$ be a matrix of $l$ columns of $\L\in\reals^{m\times n}$ 
sampled uniformly without replacement.
If $l \geq \spikiness^4(\L)\log(1/\delta)/(2\epsilon^2),$ then 
$$\displaystyle\spikiness(\L_C) \leq \frac{\spikiness(\L)}{\sqrt{1-\epsilon}}$$
with probability at least $1-\delta$.
\end{lemma}

\subsubsection{Column Projection Analysis}

Our first theorem asserts that, with high probability, column projection
produces an approximation nearly as good as a given rank-$r$ target by sampling
a number of columns proportional to the spikiness and $r\log (mn)$.  

\begin{theorem}[Column Projection under Non-Spikiness]
\label{cor:proj-main-spike}
Given a matrix $\M\in\reals^{m\times n}$ and a rank-$r$, $\spikiness$-spiky approximation
$\L\in\reals^{m\times n}$, choose 
$$l \geq 8r\spikiness^4\log(2mn/\delta)/\epsilon^2,$$
and let $\C\in\reals^{m\times l}$ be a matrix of $l$ columns of $\M$
sampled uniformly without replacement.  Then, $$\norm{\M - \Lproj}_F \leq
\norm{\M - \L}_F+\epsilon$$ with probability at least $1-\delta$,
whenever $\norm{\M}_\infty \leq \spikiness/\sqrt{mn}$. 
\end{theorem}

The proof of Thm.~\ref{cor:proj-main-spike} builds upon the randomized matrix multiplication work
of \citet{DrineasKaMa06a,DrineasKaMa06b} and will be given in Sec.~\ref{sec:proj-spike}.

\subsection{Base Algorithm Guarantee} 
The next result, a reformulation of \citet[Cor.~1]{NegahbanWa12}, is a prototypical example of a spikiness-based estimation guarantee for noisy MC.
Cor.~\ref{cor:mc-spike} bounds the estimation error of a convex optimization approach
to noisy matrix completion, under non-spikiness and uniform sampling assumptions.
\begin{corollary}[Noisy MC under Non-Spikiness~{\citep[Cor.~1]{NegahbanWa12}}]
\label{cor:mc-spike}
Suppose that $\L_0\in\reals^{m\times n}$ is $\spikiness$-spiky with rank $r$
and $\norm{\L_0}_F \leq 1$
and that $\Z_0\in\reals^{m\times n}$ has i.i.d.\ zero-mean, sub-exponential entries
with variance $\nu^2/mn$.
If, for an oversampling parameter $\beta > 0$, 
$$s\geq \spikiness^2\beta r(m+n)\log(m+n)$$ entries of $\M=\L_0+\Z_0$ are observed with locations $\obsset$
sampled uniformly with replacement, then any solution
$\hat{\L}$ of the problem
\begin{align}
\label{eqn:convex-mc-spike}
&\text{minimize}_{\L} \quad \frac{mn}{2s}\norm{\obsproj(\M-\L)}_F^2 + \lambda \norm{\L}_* \quad \text{subject\, to}\quad
\norm{\L}_\infty \leq \frac{\alpha}{\sqrt{mn}} \\ \notag
&\text{with}\quad  \lambda = 4\nu \sqrt{(m+n)\log(m+n)/s}
\end{align}
satisfies 
$$\norm{\L_0 - \hat{\L}}_F^2 \leq c_1\maxarg{\nu^2, 1} /\beta $$
with probability at least $1 - c_2\exp{-c_3\log(m+n)}$ for positive constants $c_1,c_2,$ and $c_3$.
\end{corollary}

\subsection{Spikiness Master Theorem} 

We now show that the same spikiness conditions that allow for accurate MC 
also imply high-probability estimation guarantees for \fastmf.  To make this precise, we
let $\M = \L_0 + \Z_0 \in \reals^{m \times n}$, where  $\L_0$ is
$\spikiness$-spiky with rank $r$ and 
that $\Z_0\in\reals^{m\times n}$ has i.i.d.\ zero-mean, sub-exponential entries
with variance $\nu^2/mn$.  
We further fix any $\epsilon,\delta\in(0,1]$.  
Then, our Thm.~\ref{thm:master-spike} provides a generic bound on estimation error for \fastmf when used
in combination with an arbitrary base algorithm.  The proof, which builds upon the results of
Sec.~\ref{sec:rand_mat_analysis_spike}, is deferred to Sec.~\ref{sec:proof-master_spike}.

\begin{theorem}[Spikiness Master Theorem]
\label{thm:master-spike}
Choose $t=n/l$, $l\geq 13r\spikiness^4\log(4mn/\delta)/\epsilon^2$, and
$p \geq 242\ r\log(14/\delta)/\epsilon^2$.
Under the notation of Algorithms ~$\ref{alg:fastmf-proj}$ and~$\ref{alg:fastmf-rp}$,
let $\{ \C_{0,1}, \cdots, \C_{0,t} \}$ be the corresponding partition
of $\L_0$. Then, with probability at least $1 - \delta$, \projmf and \rpmf guarantee that 
$\C_{0,i}$ is $(\sqrt{1.25}\alpha)$-spiky for all $i$ and that
\begin{align*}
&\norm{\L_0 - \Lprojmf}_F 
	\leq 2\sqrt{\textsum_{i=1}^t \norm{\C_{0,i} - \hat\C_i}_F^2} + \epsilon \quad\text{and} \\
&\norm{\L_0 - \Lrpmf}_F 
	\leq (2+\epsilon)\sqrt{\textsum_{i=1}^t \norm{\C_{0,i} - \hat\C_i}_F^2} 
\end{align*}
whenever $\norm{\hat\C_i}_\infty \leq \sqrt{1.25}\alpha/\sqrt{ml}$ for all $i$.
\end{theorem}
\begin{remark}
	The spikiness factor of $\sqrt{1.25}$ can be replaced with the smaller term $\sqrt{1+\epsilon/(4\sqrt{r})}$.
\end{remark}

When a target matrix is non-spiky, Thm.~\ref{thm:master-spike} asserts that, with high probability, the estimation error of \fastmf 
is not much larger than the error sustained by the base algorithm on each subproblem.
Thm.~\ref{thm:master-spike} further bounds the spikiness of each submatrix with high probability, 
and hence we can use any spikiness-based matrix estimation guarantee to control the estimation error on each subproblem.
The next section demonstrates how Thm.~\ref{thm:master-spike} can be applied to
derive specific \fastmf estimation guarantees for noisy MC.

\subsection{Consequences for Noisy MC}
\label{sec:cons_mc_spike}
Our corollary of Thm.~\ref{thm:master-spike} shows that \fastmf retains the
high-probability estimation guarantees of a standard MC solver while operating on
matrices of much smaller dimension.  Suppose that a base MC algorithm solves
the convex optimization problem of \eqref{eqn:convex-mc-spike}.
Then, Cor.~\ref{cor:fast-mc-spike} follows from the Spikiness Master Theorem (Thm.~\ref{thm:master-spike}) and the base algorithm guarantee of~Cor.~\ref{cor:mc-spike}.
\begin{corollary}[DFC-MC under Non-Spikiness]
\label{cor:fast-mc-spike}
Suppose that $\L_0\in\reals^{m\times n}$ is $\spikiness$-spiky with rank $r$
and $\norm{\L_0}_F \leq 1$
and that $\Z_0\in\reals^{m\times n}$ has i.i.d.\ zero-mean, sub-exponential entries
with variance $\nu^2/mn$.
Let  $c_1,c_2,$ and $c_3$ be positive constants as in Cor.~\ref{cor:mc-spike}.
If $s$ entries of $\M=\L_0+\Z_0$ are observed with locations $\obsset$
sampled uniformly with replacement, 
and the base algorithm solves the optimization problem of \eqref{eqn:convex-mc-spike},
then it suffices to choose $t=n/l$,
\begin{align*}
	l \geq
	13(c_3+1)\sqrt{\frac{(m+n)\log(m+n)\beta}{s}}nr\spikiness^4\log(4mn)/\epsilon^2,
\end{align*}
and $p \geq 242\ r\log(14(m + l)^{c_3})/\epsilon^2$ to achieve 
\begin{align*}
& \norm{\L_0 - \Lprojmf}_F \leq 2\sqrt{c_1\maxarg{({l/}{n})\nu^2, 1}/\beta} + \epsilon \quad \text{and} \\
& \norm{\L_0 - \Lrpmf}_F \leq (2 + \epsilon)\sqrt{c_1\maxarg{({l/}{n})\nu^2, 1}/\beta} 
\end{align*}
with respective probability at least
$1 - (t+1)(c_2+1)\exp{-c_3\log(m+l)}$,
if the base algorithm of \eqref{eqn:convex-mc-spike} is used
with $ \lambda = 4\nu \sqrt{(m+n)\log(m+n)/s}$.
\end{corollary}
\begin{remark}
Cor.~\ref{cor:fast-mc-spike} allows for the fraction of columns sampled to decrease as the number of revealed entries, $s$, increases.
Only a vanishingly small fraction of columns ($l/n\to 0$) need be sampled whenever 
$s=\omega((m+n)\log^3(m+n))$.
\end{remark}

To understand the conclusions of Cor.~\ref{cor:fast-mc-spike}, consider the base 
algorithm of Cor.~\ref{cor:mc-spike}, which, when applied to $\M$, recovers an estimate
$\hat{\L}$ satisfying $\norm{\L_0 - \hat{\L}}_F \leq \sqrt{c_1\maxarg{\nu^2, 1} /\beta}$ with
high probability.  
Cor.~\ref{cor:fast-mc-noise} asserts that, with appropriately reduced probability, 
\rpmf exhibits the same estimation error scaled by an
adjustable factor of $2+\epsilon$, while \projmf exhibits at most twice this error plus an 
adjustable factor of $\epsilon$.
Hence, \fastmf can quickly provide near-optimal estimation for non-spiky matrices as well as incoherent matrices, even when entries are missing.  The proof of  Cor.~\ref{cor:fast-mc-spike} can be found in Sec.~\ref{sec:fast-mc-spike}.

\section{Experimental Evaluation}
\label{sec:experiments}

We now explore the accuracy and speed-up of \fastmf on a variety of simulated
and real-world datasets.  
We use the Accelerated Proximal Gradient (APG) algorithm of \citet{TohYu10} as
our base noisy MC algorithm\footnote{Our experiments with the Augmented Lagrange Multiplier (ALM) algorithm of
\citet{LinChWuMa09} as a base algorithm (not reported) yield comparable relative speedups and performance for \fastmf.} and
the APG algorithm of \citet{LinGaWrWuChMa09} as our base noisy RMF algorithm.
We perform all experiments on an x86-64 architecture using a single 2.60 Ghz
core and 30GB of main memory. We use the default parameter settings suggested
by \citet{TohYu10} and \citet{LinGaWrWuChMa09}, and measure estimation error
via root mean square error (RMSE).  To achieve a fair running time comparison,
we execute each subproblem in the F step of \fastmf in a serial fashion on the same machine
using a single core.  
Since, in practice, each of these subproblems would be executed in parallel, 
the \emph{parallel running time} of \fastmf is calculated as the time to complete the D and C steps of DFC
plus the running time of the longest running subproblem in the F step.
We compare \fastmf to two baseline methods: the base algorithm APG applied to the full matrix $\M$ and \partmf, which 
carries out the D and F steps of \projmf but omits the final C step (projection).
\subsection{Simulations}

For our simulations, we focused on square matrices ($m=n)$ and generated random
low-rank and sparse decompositions, similar to the schemes used in related
work \citep{CandesLiMaWr09, KeshavanMoOh10b,ZhouLiWrCaMa10}.  We created
$\L_0\in\reals^{m\times m}$ as a random product, $\A\B^\top$, where $\A$ and
$\B$ are $m\times r$ matrices with independent $\Gsn(0,\sqrt{1/r})$ entries
such that each entry of $\L_0$ has unit variance.  $\Z_0$ contained independent
$\Gsn(0,0.1)$ entries. In the MC setting, $s$ entries of $\L_0 + \Z_0$ were
revealed uniformly at random. In the RMF setting, the support of $\S_0$ was
generated uniformly at random, and the $s$ corrupted entries took values in
$[0,1]$ with uniform probability.  For each algorithm, we report error between
$\L_0$ and the estimated low-rank matrix, and all reported results are averages
over ten trials. 

\begin{figure*}[ht!]
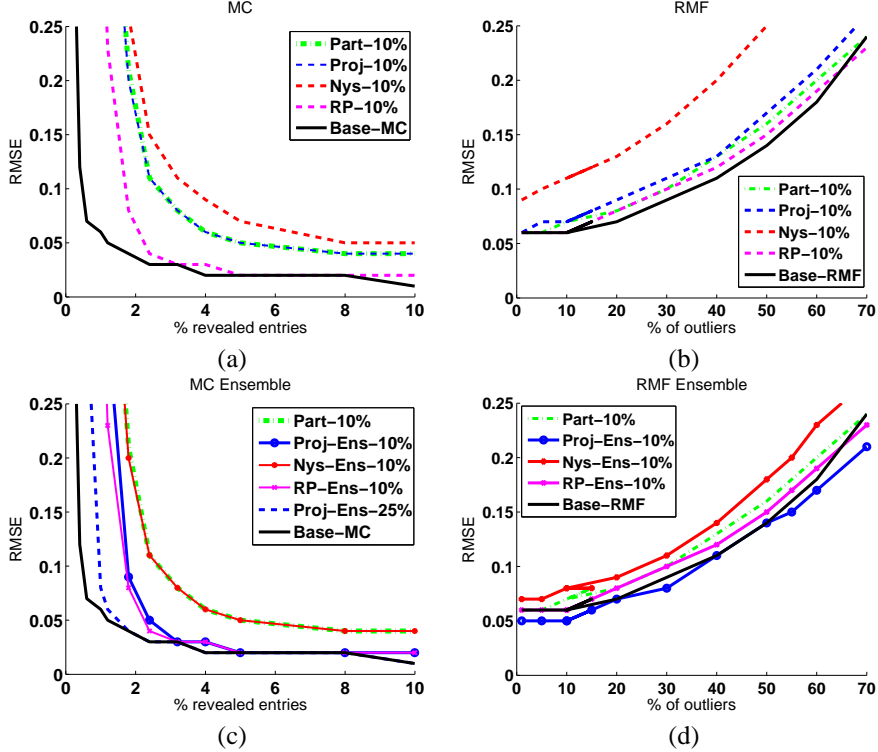

\begin{center}
\begin{tabular} {@{}c@{}c@{}}
\ipsfig{.38}{figure=exper1_noisy_mc_noens.eps} &  
\ipsfig{.38}{figure=exper1_noisy_rmf_noens.eps} \\ 
(a) & (b) \\
\ipsfig{.38}{figure=exper1_noisy_mc_ens.eps} &  
\ipsfig{.38}{figure=exper1_noisy_rmf_ens.eps} \\ 
(c) & (d) \\
\end{tabular}
\end{center}
\vspace{-3mm}
\caption{Recovery error of \fastmf relative to base algorithms.}
\label{fig:exper_sim1}
\end{figure*} 

We first explored the estimation error of \fastmf as a function of $s$, using
($m=10$K, $r=10)$ with varying observation sparsity for MC and ($m=1$K, $r=10)$
with a varying percentage of outliers for RMF.  The results are summarized in
Figure~\ref{fig:exper_sim1}.
In both MC and RMF, the gaps in estimation between APG and \fastmf
are small when sampling only 10\% of rows and columns.  Moreover, of the
standard \fastmf algorithms, \rpmf performs the best, as shown in
Figures~\ref{fig:exper_sim1}(a) and (b). Ensembling improves the performance
of \gnysmf and \projmf, as shown in Figures~\ref{fig:exper_sim1}(c) and (d),
and \projmfens in particular consistently outperforms \partmf and \gnysmfens,
slightly outperforms \rpmf, and matches the performance of APG for most
settings of $s$.  In practice we observe that $\Lrp$ equals the optimal
(with respect to the spectral or Frobenius norm) rank-$k$ approximation of
$[\hat \C_1,\ldots,\hat \C_t]$, and thus the performance of \rpmf consistently
matches that of \rpmfens.  We therefore omit the \rpmfens results in the remainder this section.

We next explored the speed-up of \fastmf as a function of matrix size. For MC,
we revealed $4\%$ of the matrix entries and set $r = 0.001 \cdot m$, while for
RMF we fixed the percentage of outliers to $10\%$ and set $r = 0.01 \cdot m$.
We sampled $10\%$ of rows and columns and observed that estimation errors were
comparable to the errors presented in Figure~\ref{fig:exper_sim1} for similar
settings of $s$; in particular, at all values of $n$ for both MC and RMF, the
errors of APG and \projmfens were nearly identical.  Our timing results,
presented in Figure \ref{fig:exper_sim2}, illustrate a near-linear speed-up for
MC and a superlinear speed-up for RMF across varying matrix sizes.
Note that the timing curves of the \fastmf algorithms and \partmf all overlap,
a fact that highlights the minimal computational cost of the
final matrix approximation step.

\begin{figure*}[ht!]
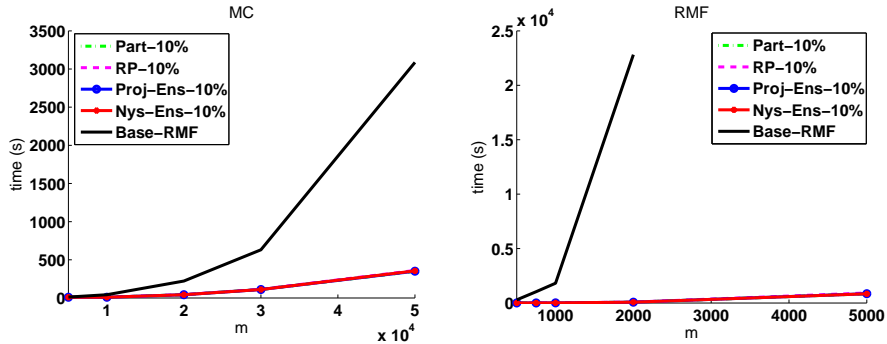

\begin{center}
\begin{tabular} {@{}c@{}c@{}}
\ipsfig{.38}{figure=exper2_noisy_mc_journal.eps} &  
\ipsfig{.38}{figure=exper2_noisy_rmf_journal.eps} \\ 
\end{tabular}
\end{center}
\vspace{-4mm}
\caption{Speed-up of \fastmf relative to base algorithms.}
\vspace{-5mm}
\label{fig:exper_sim2}
\end{figure*} 

\subsection{Collaborative Filtering}
\label{sec:cf}
Collaborative filtering for recommender systems is one prevalent real-world
application of noisy matrix completion. A collaborative filtering dataset can
be interpreted as the incomplete observation of a ratings matrix with columns
corresponding to users and rows corresponding to items. The goal is to infer
the unobserved entries of this ratings matrix. We evaluate \fastmf on two of
the largest publicly available collaborative filtering datasets: MovieLens
10M\footnote{\url{http://www.grouplens.org/}} ($m = 4$K, $n=6$K, $s>10$M) and
the Netflix Prize dataset\footnote{\url{http://www.netflixprize.com/}} ($m =
18$K, $n=480$K, $s>100$M).  To generate test sets drawn from the training
distribution, for each dataset, we aggregated all available rating data into a
single training set and withheld test entries uniformly at random, while
ensuring that at least one training observation remained in each row and
column.  The algorithms were then run on the remaining training portions and
evaluated on the test portions of each split.  The results, averaged over three
train-test splits, are summarized in Table~\ref{tab:cf}. Notably, \projmf,
\projmfens, \gnysmfens, and \rpmf all outperform \partmf, and \projmfens performs
comparably to APG while providing a nearly linear parallel time speed-up.
Similar to the simulation results presented in Figure~\ref{fig:exper_sim1},
\rpmf performs the best of the standard \fastmf algorithms, though \projmfens
slightly outperforms \rpmf.  Moreover, the poorer performance of \gnysmf can be
in part explained by the asymmetry of these problems.  Since these matrices
have many more columns than rows, MF on column submatrices is inherently easier
than MF on row submatrices, and for \gnysmf, we observe that $\hat \C$ is an
accurate estimate while $\hat \R$ is not.  
\begin{table*}[ht!]
\ifdefined\jacmformat
\tbl{Performance of \fastmf relative to base algorithm APG on collaborative filtering tasks.\label{tab:cf}}{
\else
	\caption{Performance of \fastmf relative to base algorithm APG on collaborative filtering tasks.}
	\label{tab:cf}
  \begin{center}
\fi
  \begin{tabular}{lcccc}
    \hline    
		\multirow{2}{*}{\bf Method }&\multicolumn{2}{c}{{\bf MovieLens 10M}}
    &\multicolumn{2}{c}{{\bf Netflix}} \\
    & {\small\bf RMSE} & {\small\bf Time} & {\small\bf RMSE} & {\small\bf Time} \\
    \hline
		\\[-.3cm]
		\textbf{Base algorithm (APG)} & \textbf{0.8005} & \textbf{552.3s} & \textbf{0.8433}	& \textbf{4775.4s} \\
		\\[-.2cm]
		\partmf-25\% & 0.8146 & 146.2s & 0.8451 & 1274.6s \\
		\partmf-10\% & 0.8461 & 56.0s & 0.8491 & 548.0s \\
		\\[-.2cm]
		\gnysmf-25\% & 0.8449 & 141.9s & 0.8832 & 1541.2s \\
		\gnysmf-10\% & 0.8776 & 82.5s & 0.9228 & 797.4s \\
		\\[-.2cm]
    \gnysmfens-25\% & 0.8085 & 153.5s & 0.8486 & 1661.2s \\
    \gnysmfens-10\% & 0.8328 & 96.2s & 0.8613 & 909.8s \\		
    \\[-.2cm]
		\projmf-25\% & 0.8061 & 146.3s & 0.8436 & 1274.8s \\
		\projmf-10\% & 0.8270 & 56.0s & 0.8486 & 548.1s \\
		\\[-.2cm]
		\textbf{\projmfens-25\%} & \textbf{0.7944} & \textbf{146.3s} & \textbf{0.8411} & \textbf{1274.8s} \\
		\textbf{\projmfens-10\%} & \textbf{0.8117} & \textbf{56.0s} & \textbf{0.8434} & \textbf{548.1s} \\
    \\[-.2cm]
		\rpmf-25\% & 0.8027 & 147.4s & 0.8438 & 1283.6s \\
		\rpmf-10\% & 0.8074 & 56.2s & 0.8448 & 550.1s \\
		\hline
	\end{tabular}
\ifdefined\jacmformat
	 }
\else
	\end{center}
\fi
\end{table*}

\subsection{Background Modeling in Computer Vision}
Background modeling has important practical ramifications for detecting
activity in surveillance video.  This problem can be framed as an application
of noisy RMF, where each video frame is a column of some  matrix ($\M)$, the
background model is low-rank ($\L_0)$, and moving objects and background
variations, e.g., changes in illumination, are outliers ($\S_0)$.  We evaluate
\fastmf on two videos: `Hall' ($200$ frames of size $176\times 144)$ contains
significant foreground variation and was studied by \citet{CandesLiMaWr09},
while `Lobby' ($1546$ frames of size $168 \times 120)$ includes many changes in
illumination (a smaller video with $250$ frames was studied by
\citet{CandesLiMaWr09}).  We focused on \projmfens, due to its superior
performance in previous experiments, and measured the RMSE between the
background model estimated by \fastmf and that of APG.  On both videos,
\projmfens estimated nearly the same background model as the full APG algorithm
in a small fraction of the time.  On `Hall,' the \projmfens-5\% and
\projmfens-0.5\% models exhibited RMSEs of $0.564$ and $1.55$, quite small
given pixels with $256$ intensity values.  The associated running time was reduced
from $342.5$s for APG to real-time ($5.2$s for a $13$s video) for
\projmfens-0.5\%. Snapshots of the results are presented in
Figure~\ref{fig:rmf_hall}.  On `Lobby,' the RMSE of \projmfens-4\% was $0.64$,
and the speed-up over APG was more than 20X, i.e., the running time reduced from
$16557$s to $792$s.   
\begin{figure*}[t!]
\begin{center}
\begin{tabular} {@{}c@{}|c@{}c@{}c@{}}
\ipsfig{.22}{figure=hall_orig_f20.eps} &  
\ipsfig{.22}{figure=hall_apg_f20.eps} &
\ipsfig{.22}{figure=hall_projRens_1_f20.eps} &  
\ipsfig{.22}{figure=hall_projRens_005_f20.eps} \\  
Original frame & APG & 5\% sampled & 0.5\% sampled \\
 & (342.5s) & (24.2s) & (5.2s)\\
\end{tabular}
\end{center}
\caption{Sample `Hall' estimation by APG, \projmfens-5\%, and
\projmfens-.5\%.}
\label{fig:rmf_hall}
\end{figure*}

\section{Conclusions}
\label{sec:conclusion}
To improve the scalability of existing matrix factorization algorithms while
leveraging the ubiquity of parallel computing architectures, we introduced,
evaluated, and analyzed \fastmf, a divide-and-conquer framework for noisy
matrix factorization with missing entries or outliers.  
\fastmf is trivially parallelized and particularly well suited for distributed
environments given its low communication footprint.  Moreover, \fastmf provably
maintains the estimation guarantees of its base algorithm, even in the presence
of noise, and yields linear to super-linear speedups in practice. 

A number of natural follow-up questions suggest themselves.
First, can the sampling complexities and conclusions of our theoretical analyses be strengthened?
For example, can the $(2+\epsilon)$ approximation guarantees of our master theorems be sharpened to 
$(1+\epsilon)$?
Second, how does \fastmf perform when paired with alternative base algorithms, having no theoretical guarantees but
displaying other practical benefits?   These open questions are fertile ground for future work.
\appendix

\section{Proof of Theorem~\lowercase{\ref{thm:regress-main}}: Subsampled Regression under Incoherence}
\label{sec:regress-proof}
We now give a proof of Thm.~\ref{thm:regress-main}.  While the results of this
section are stated in terms of i.i.d.\ with-replacement sampling of
columns and rows, a concise argument due to~\citet[Sec.~6]{Hoeffding63} implies
the same conclusions when columns and rows are sampled without replacement.

Our proof of Thm.~\ref{thm:regress-main} will require a strengthened version of the
randomized $\ell_2$ regression work of \citet[Thm.~5]{DrineasMaMu08}.  The proof
of Thm.~5 of \citet{DrineasMaMu08} relies heavily on the fact that
$\norm{\A\B-\G\H}_F\leq\frac{\epsilon}{2}\norm{\A}_F\norm{\B}_F$ with
probability at least 0.9, when $\G$ and $\H$ contain sufficiently many rescaled
columns and rows of $\A$ and $\B$, sampled according to a particular
non-uniform probability distribution.  A result of~\citet{HsuKaZh12}, modified
to allow for slack in the probabilities, establishes a related claim with improved
sampling complexity.\footnote{The general conclusion of \citep[Example~4.3]{HsuKaZh12} is incorrectly stated as noted in \cite{Hsu12}.  However, the original statement is correct in the special case when a matrix is multiplied by its own transpose, which is the case of interest here.}

\begin{lemma}[{\citep[Example~4.3]{HsuKaZh12}}]
\label{lem:mat-mult}
Given a matrix $\A\in\reals^{m\times k}$ with
$r\geq \rank{\A}$, an error tolerance $\epsilon\in(0,1]$,
and a failure probability $\delta\in(0,1]$, define probabilities $p_j$
satisfying
\begin{align}
p_j&\geq \frac{\beta}{Z}\norm{\A_{(j)}}^2,
\quad Z = \sum_j\norm{\A_{(j)}}^2,
\quad \text{and}\quad \textsum_{j=1}^kp_j = 1
\end{align}
for some $\beta\in(0,1]$.  Let $\G\in\reals^{m\times l}$ be a column submatrix
of $\A$ in which exactly $l\geq 48r\log(4r/(\beta\delta))/(\beta\epsilon^2)$ columns are
selected in i.i.d.\ trials in which the $j$-th column is chosen with probability
$p_j$.  Further, let $\D\in\reals^{l\times l}$ be a
diagonal rescaling matrix with entry $\D_{tt} = 1/\sqrt{lp_j}$ whenever the
$j$-th column of $\A$ is selected on the $t$-th sampling trial, for $t =
1,\ldots,l$.  Then, with probability at least $1-\delta$, 
\begin{equation*}
\norm{\A\A^\top - \G\D\D\G^\top}_2 \leq \frac{\epsilon}{2}\norm{\A}_2^2.
\end{equation*}
\end{lemma}

Using Lem.~\ref{lem:mat-mult}, we now establish a stronger version of Lem. 1
of~\citet{DrineasMaMu08}.  For a given $\beta\in (0,1]$ and
$\L\in\reals^{m\times n}$ with rank $r$, we first define column sampling
probabilities $p_j$ satisfying
\begin{align}
\label{eq:cur-col-prob}
p_j&\geq \frac{\beta}{r}\norm{(\V_{L})_{(j)}}^2 \quad \text{and}\quad
\textsum_{j=1}^np_j = 1.
\end{align}
We further let $\S\in\reals^{n\times l}$ be a random binary matrix with
independent columns, where a single 1 appears in each column, and $\S_{jt} = 1$
with probability $p_j$ for each $t\in\{1,\ldots,l\}$.  Moreover, let
$\D\in\reals^{l\times l}$ be a diagonal rescaling matrix with entry $\D_{tt} =
1/\sqrt{lp_j}$ whenever $\S_{jt} = 1$.  Postmultiplication by $\S$ is
equivalent to selecting $l$ random columns of a matrix, independently and with
replacement.  Under this notation, we establish the following lemma:
\begin{lemma}
\label{lem:eps}
Let $\epsilon \in (0, 1],$ and define $\V_l^\top = \V_{L}^\top\S$ and $\Gamma =
(\V_l^\top\D)^+ - (\V_l^\top\D)^\top$.  If $l \geq
48r\log(4r/(\beta\delta))/(\beta\epsilon^2)$ for $\delta\in(0,1]$ then with probability
at least $1-\delta$:
\begin{gather*}
\rank{\V_l} = \rank{\V_L} = \rank{\L}\\
\norm{\Gamma}_2 = \norm{\mSigma_{V_l^\top D}^{-1}-\mSigma_{V_l^\top D}}_2\\
(\L\S\D)^+ = (\V_l^\top\D)^+\mSigma_{L}^{-1}\U_L^\top\\
\norm{\mSigma_{V_l^\top D}^{-1}-\mSigma_{V_l^\top D}}_2 \leq \epsilon/\sqrt{2}.
\end{gather*}
\end{lemma}
\begin{proof}
By Lem.~\ref{lem:mat-mult}, for all $1\leq i \leq r$,
\begin{align*}
|1-\sigma_i^2(\V_l^\top\D)| &= |\sigma_i(\V_L^\top
\V_L)-\sigma_i(\V_l^\top\D\D\V_l)|\\
&\leq \norm{\V_L^\top \V_L-\V_{L}^\top\S\D \D\S^\top\V_L}_2\\
&\leq \epsilon/2\norm{\V_L^\top}_2^2 = \epsilon/2,
\end{align*}
where $\sigma_i(\cdot)$ is the $i$-th largest singular value of a given matrix.
Since $\epsilon/2\leq 1/2$, each singular value of $\V_l$ is positive, and so
$\rank{\V_l} = \rank{\V_L} = \rank{\L}$.  The remainder of the proof is
identical to that of Lem.~1 of \citet{DrineasMaMu08}.
\end{proof}

Lem.~\ref{lem:eps} immediately yields improved sampling complexity for the
randomized $\ell_2$ regression of \citet{DrineasMaMu08}:
\begin{proposition}
\label{prop:regress}
Suppose $\B\in\reals^{p\times n}$ and $\epsilon\in(0,1]$.  If $l\geq
3200r\log(4r/(\beta\delta))/(\beta\epsilon^2)$ for $\delta\in(0,1]$, then with
probability at least $1-\delta-0.2$: $$\norm{\B-\B\S\D(\L\S\D)^+\L}_F \leq
(1+\epsilon)\norm{\B-\B\L^+\L}_F.$$
\end{proposition}
\begin{proof}
The proof is identical to that of Thm.~5 of \citet{DrineasMaMu08} once
Lem.~\ref{lem:eps} is substituted for Lem.~1 of \citet{DrineasMaMu08}.
\end{proof}

A typical application of Prop.~\ref{prop:regress} would involve performing a
truncated SVD of $\M$ to obtain the \emph{statistical leverage scores},
$\norm{(\V_{L})_{(j)}}^2$, used to compute the column sampling probabilities of
\eqref{eq:cur-col-prob}.  Here, we will take advantage of the slack term,
$\beta$, allowed in the sampling probabilities of \eqref{eq:cur-col-prob} to
show that uniform column sampling gives rise to the same estimation guarantees
for column projection approximations when $\L$ is sufficiently incoherent.

To prove Thm.~\ref{thm:regress-main}, we first notice that $n\geq r\mu_0(\V_{L})$
and hence
\begin{align*}
	l &\geq 3200r\mu_0(\V_{L})\log(4r \mu_0(\V_{L})/\delta)/\epsilon^2 \\
	&\geq 3200r\log(4r/(\beta\delta))/(\beta\epsilon^2)
\end{align*}
whenever $\beta\geq1/\mu_0(\V_{L})$.
Thus,  we may apply Prop.~\ref{prop:regress} with 
$\beta=1/\mu_0(\V_{L})\in (0,1]$ and $p_j = 1/n$ by noting that 
$$\frac{\beta}{r} \norm{(\V_L)_{(j)}}^2 \leq
\frac{\beta}{r}\frac{r}{n}\mu_0(\V_{L}) =  \frac{1}{n} = p_j $$ for all $j$, by
the definition of $\mu_0(\V_{L})$.  By our choice of probabilities, $\D =
\I\sqrt{n/l}$, and hence $$\norm{\B - \B_C\pinv{\L_C}\L}_F =\norm{\B -
\B_C\D\pinv{(\L_C\D)}\L}_F \leq (1+\epsilon)\norm{\B - \B\pinv{\L}\L}_F$$ with probability at
least $1-\delta - 0.2$, as desired.

\section{Proof of Lemma~\lowercase{\ref{lem:sub-coh}}: Conservation of Incoherence}
\label{sec:sub-coh}
Since for all $n>1$, 
$$c\log(n)\log(1/\delta)= ({c}/{4})\log(n^4)\log(1/\delta)
\geq 48\log(4n^2/\delta)\geq 48\log(4r\mu_0(\V_{L})/(\delta/n))$$ 
as $n\geq r\mu_0(\V_{L})$, 
claim $i$ follows immediately from Lemma~\ref{lem:eps} with
$\beta=1/\mu_0(\V_{L})$, $p_j=1/n$ for all $j$, 
and $\D = \I\sqrt{n/l}$.  
When $\rank{\L_C} = \rank{\L}$, Lemma~1 of~\citet{MohriTa11}
implies that $\P_{U_{L_C}} = \P_{U_L}$, which in turn implies claim $ii$.

To prove claim $iii$ given the conclusions of Lemma~\ref{lem:eps}, assume,
without loss of generality, that $\V_l$ consists of the first $l$ rows of
$\V_{L}$.  Then if $\L_C = \U_{L}\mSigma_{L}\V_l^\top$ has
$\rank{\L_C}=\rank{\L}=r$, the matrix $\V_l$ must have full column rank.  Thus
we can write
\begin{align*}
\L_C^+\L_C &= (\U_{L}\mSigma_{L}\V_l^\top)^+\U_{L}\mSigma_{L}\V_l^\top\\
&= (\mSigma_{L}\V_l^\top)^+\U_{L}^+\U_{L}\mSigma_{L}\V_l^\top\\
&= (\mSigma_{L}\V_l^\top)^+\mSigma_{L}\V_l^\top\\
&= (\V_l^\top)^+\mSigma_{L}^+\mSigma_{L}\V_l^\top\\
&= (\V_l^\top)^+\V_l^\top\\
&= \V_l(\V_l^\top\V_l)^{-1}\V_l^\top,
\end{align*}
where the second and third equalities follow from $\U_{L}$ having orthonormal
columns, the fourth and fifth result from $\mSigma_{L}$ having full rank and
$\V_l$ having full column rank, and the sixth follows from $\V_l^\top$ having
full row rank.

Now, denote the right singular vectors of $\L_C$ by $\V_{L_C} \in
\reals^{l\times r}$. Observe that $\P_{V_{L_C}} = \V_{L_C} \V_{L_C}^\top =
\L_C^+\L_C$, and define $\evec_{i,l}$ as the $i$th column of $\I_l$ and
$\evec_{i,n}$ as the $i$th column of $\I_n$. Then we have,
\begin{align*}
\mu_0(\V_{L_C}) &= \frac{l}{r} \max_{1 \le i \le l} \norm{\P_{V_{L_C}} \evec_{i,l}}^2 \\
&= \frac{l}{r} \max_{1 \le i \le l} \evec_{i,l}^\top\L_C^+\L_C\evec_{i,l} \\
&= \frac{l}{r} \max_{1 \le i \le l} \evec_{i,l}^\top(\V_l^\top)^+\V_l^\top\evec_{i,l}\\
&= \frac{l}{r} \max_{1 \le i \le l} \evec_{i,l}^\top\V_l(\V_l^\top\V_l)^{-1}\V_l^\top\evec_{i,l}\\
&= \frac{l}{r} \max_{1 \le i \le l} \evec_{i,n}^\top\V_{L}(\V_l^\top\V_l)^{-1}\V_{L}^\top\evec_{i,n},
\end{align*}
where the final equality follows from $\V_l^\top \evec_{i,l} = \V_{L}^\top
\evec_{i,n}$ for all $1\le i\le l$.

Now, defining $\Q = \V_l^\top\V_l$ we have
\begin{align*}
\mu_0(\V_{L_C}) &= \frac{l}{r} \max_{1 \le i \le l}
\evec_{i,n}^\top\V_{L}\Q^{-1}\V_{L}^\top\evec_{i,n} \\
&= \frac{l}{r} \max_{1 \le i \le l}
\Trarg{\evec_{i,n}^\top\V_{L}\Q^{-1}\V_{L}^\top\evec_{i,n}} \\
&= \frac{l}{r} \max_{1 \le i \le l} \Trarg{\Q^{-1}\V_{L}^\top\evec_{i,n}
\evec_{i,n}^\top\V_{L}}\\
&\le \frac{l}{r} \norm{\Q^{-1}}_2 \max_{1 \le i \le l}
\norm{\V_{L}^\top\evec_{i,n} \evec_{i,n}^\top\V_{L}}_*\ ,
\end{align*}
by H\"older's inequality for Schatten $p$-norms.
Since $\V_{L}^\top\evec_{i,n} \evec_{i,n}^\top\V_{L}$ has rank one, 
we can explicitly compute its trace norm as $\norm{\V_{L}^\top\evec_{i,n}}^2=\norm{\P_{V_{L}}\evec_{i,n}}^2$.
Hence,
\begin{align*}
\mu_0(\V_{L_C}) &\leq \frac{l}{r} \norm{\Q^{-1}}_2 \max_{1 \le i \le l}
\norm{\P_{V_{L}}\evec_{i,n}}^2\\
&\le \frac{l}{r}\frac{r}{n} \norm{\Q^{-1}}_2 \bigg( \frac{n}{r}\max_{1 \le i
\le n} \norm{\P_{V_{L}}\evec_{i,n}}^2\bigg)\\
&= \frac{l}{n} \norm{\Q^{-1}}_2 \mu_0(\V_{L})\, ,
\end{align*}
by the definition of $\mu_0$-coherence.
The proof of Lemma~\ref{lem:eps} established that the
smallest singular value of $\frac{n}{l}\Q = \V_l^\top\D\D\V_l$ is lower bounded
by $1-\frac{\epsilon}{2}$ and hence $\norm{\Q^{-1}}_2 \leq
\frac{n}{l(1-\epsilon/2)}$.  Thus, we conclude that $\mu_0(\V_{L_C}) \leq
\mu_0(\V_{L})/(1-\epsilon/2)$.

To prove claim $iv$ under Lemma~\ref{lem:eps}, we note that 
\begin{align*}
\mu_1(\L_C) &= \sqrt{\frac{ml}{r}}\max_{\substack{1\le i \le m\\ 1\le j \le
l}}|\evec_{i,m}^\top\U_{L_C}\V_{L_C}^\top \evec_{j,l}| \\
&\leq \sqrt{\frac{ml}{r}} \max_{1\le i \le m} \norm{\U_{L_C}^\top \evec_{i,m}} \max_{1\le j \le l}\norm{\V_{L_C}^\top \evec_{j,l}} \\
&= \sqrt{r} \bigg( \sqrt{\frac{m}{r}}\max_{1\le i \le m }\norm{\P_{U_{L_C}}\evec_{i,m}}\bigg)
\bigg( \sqrt{\frac{l}{r}}\max_{1\le j \le l }\norm{\P_{V_{L_C}}\evec_{j,l}} \bigg)\\
&= \sqrt{r\mu_0(\U_{L_C})\mu_0(\V_{L_C})} 
\leq \sqrt{r\mu_0(\U_{L})\mu_0(\V_{L})/(1-\epsilon/2)}  
\end{align*}
by H\"older's inequality for Schatten $p$-norms, the definition of $\mu_0$-coherence, and claims $ii$ and $iii$.

\section{Proof of Corollary~\lowercase{\ref{cor:proj-main}}: Column Projection under Incoherence}
\label{sec:proj-proof}
Fix $c = 48000/\log(1/0.45)$, and notice that for $n > 1$,
$$48000\log(n) \geq 3200\log(n^5) \geq 3200\log(16n).$$
Hence $l \geq 3200r\mu_0(\V_{L})\log(16n)(\log(\delta)/\log(0.45))/\epsilon^2.$

Now partition the columns of $\C$ into $b = \log(\delta)/\log(0.45)$ 
submatrices, $\C=[\C_1,\cdots,\C_b]$, each with $a=l/b$ columns,\footnote{For simplicity, 
we assume that $b$ divides $l$ evenly.}
and let $[\L_{C_1},\cdots,\L_{C_b}]$ be the corresponding partition of $\L_C$.
Since 
$$a \geq 3200r\mu_0(\V_{L})\log(4n/0.25)/\epsilon^2,$$
we may apply Prop.~\ref{prop:regress} independently for each $i$ to yield
\begin{equation}
\label{eqn:subblock-succ}
\norm{\M - \C_i\pinv{\L_{C_i}}\L}_F \leq (1+\epsilon)\norm{\M - \M\pinv{\L}\L}_F
\leq (1+\epsilon)\norm{\M - \L}_F
\end{equation}
with probability at least $0.55$, since $\M\pinv{\L}$ minimizes
$\norm{\M-\Y\L}_F$ over all $\Y\in\reals^{m\times m}$.

Since each $\C_i = \C\S_i$ for some matrix $\S_i$ and $\pinv{\C}\M$ minimizes $\norm{\M-\C\X}_F$ over all
$\X\in\reals^{l\times n}$, it follows that
$$\norm{\M-\C\pinv{\C}\M}_F\leq\norm{\M-\C_i\pinv{\L_{C_i}}\L}_F,$$
for each $i$.
Hence, if 
$$\norm{\M - \C\C^+\M}_F \leq (1+\epsilon)\norm{\M - \L}_F,$$
fails to hold, then, for each $i$, \eqref{eqn:subblock-succ} also fails to hold.
The desired conclusion therefore must hold with probability at least
$1-0.45^b = 1-\delta$.

\section{Proof of Corollary~\lowercase{\ref{cor:gnys-main}}: Generalized \nys Method under Incoherence}
\label{sec:gnys-proof}
With $c = 48000/\log(1/0.45)$ as in Cor.~\ref{cor:proj-main}, we notice that for $m > 1$,
$$48000\log(m) = 16000\log(m^3) \geq 16000\log(4m).$$
Therefore, 
\begin{align*}
	d &\geq 16000r\mu_0(\U_{C})\log(4m)(\log(\delta')/\log(0.45))/\epsilon^2 \\
	   &\geq 3200r\mu_0(\U_{C})\log(4m/\delta')/\epsilon^2,
\end{align*}
for all $m > 1$ and $\delta' \leq 0.8$.
Hence, we may apply Thm.~\ref{thm:regress-main} and Cor.~\ref{cor:proj-main} in turn
to obtain
$$\norm{\M - \C\pinv{\W}\R}_F \leq (1+\epsilon)\norm{\M - \C\pinv{\C}\M}_F
\leq (1+\epsilon)^2\norm{\M - \L}$$
with probability at least $(1-\delta)(1-\delta' - 0.2)$ by independence.  

\section{Proof of Corollary~\lowercase{\ref{cor:gnys-low-rank}}: Noiseless Generalized \nys Method under Incoherence}
\label{sec:gnys-low-rank-proof}
Since $\rank{\L}=r$, $\L$ admits a decomposition
$\L = \Y^{\top}\Z$ for some matrices $\Y\in\reals^{r\times m}$ and
$\Z\in\reals^{r\times n}$.  In particular, let $\Y^{\top} =
\U_{L}\mSigma_{L}^{\half}$ and $\Z = \mSigma_{L}^{\half}\V_{L}^{\top}$.  By block
partitioning $\Y$ and $\Z$ as $\Y = \begin{bmatrix}\Y_1 & \Y_2\end{bmatrix}$
and $\Z = \begin{bmatrix}\Z_1 & \Z_2\end{bmatrix}$ for $\Y_1\in\reals^{r\times
d}$ and $\Z_1\in\reals^{r\times l}$, we may write $\W = \Y_1^{\top}\Z_1, \C =
\Y^{\top}\Z_1,$ and $\R = \Y_1^{\top}\Z$.  Note that we assume that the
generalized \nys approximation is generated from sampling the first $l$ columns
and the first $d$ rows of $\L$, which we do without loss of generality since
the rows and columns of the original low-rank matrix can always be permutated to match
this assumption. 

Prop.~\ref{prop:exact} shows that, like the \nys method~\citep{KumarMoTa09}, the
generalized \nys method yields exact recovery of $\L$ whenever $\rank{\L} =
\rank{\W}$.  The same result was established in~\citet{WangDoToLiGu09} with a
different proof.
\begin{proposition}
	\label{prop:exact}
	Suppose $r = \rank{\L} \leq \min(d,l)$ and $\rank{\W} = r$.
	Then $\L = \Lgnys$.
\end{proposition}
\begin{proof}
By appealing to our factorized block decomposition, we may rewrite the
generalized \nys approximation as
$\Lgnys = \C\W^+\R = \Y^{\top}\Z_1(\Y_1^{\top}\Z_1)^+\Y_1^{\top}\Z$.
We first note that 
$\rank{\W}=r$ implies that $\rank{\Y_1}=r$ and
$\rank{\Z_1}=r$ so that $\Z_1\Z_1^{\top}$ and $\Y_1\Y_1^{\top}$ are full-rank.
Hence, $(\Y_1^{\top}\Z_1)^+ =
\Z_1^{\top}(\Z_1\Z_1^{\top})^{-1}(\Y_1\Y_1^{\top})^{-1}\Y_1,$ yielding
\begin{equation*}
	\Lgnys =
\Y^{\top}\Z_1\Z_1^{\top}(\Z_1\Z_1^{\top})^{-1}(\Y_1\Y_1^{\top})^{-1}\Y_1\Y_1^{\top}\Z
= \Y^{\top}\Z = \L.
\end{equation*}
\end{proof}

Prop.~\ref{prop:exact} allows us to lower bound the probability of exact
recovery with the probability of randomly selecting a rank-$r$ submatrix.  As
$\rank{\W}=r$ iff both $\rank{\Y_1}=r$ and $\rank{\Z_1}=r$, it suffices to
characterize the probability of selecting full rank submatrices of $\Y$ and
$\Z$.  Following the treatment of the \nys method in~\citet{TalwalkarRo10}, we
note that $\mSigma_{L}^{-\half}\Z = \V_{L}^{\top}$ and hence that
$\Z_1^{\top}\mSigma_{\L}^{-\half} = \V_{l}$ where $\V_{l}\in\reals^{l\times r}$
contains the first $l$ components of the leading $r$ right singular vectors of
$\L$.  It follows that $\rank{\Z_1} = \rank{\Z_1^{\top}\mSigma_{L}^{-\half}} =
\rank{\V_{l}}$.  Similarly, $\rank{\Y_1} =
\rank{\U_{d}}$ where $\U_{d}\in\reals^{d\times r}$ contains
the first $d$ components of the leading $r$ left singular vectors of $\L$.
Thus, we have 
\begin{align}
	\label{eq:rankprobZ}
	&\Parg{\rank{\Z_1}=r} = \Parg{\rank{\V_{l}}=r} \qquad \text{and}\\
	\label{eq:rankprobY}
	&\Parg{\rank{\Y_1}=r} = \Parg{\rank{\U_{d}}=r}.
\end{align}

Next we can apply the first result of Lem.~\ref{lem:eps} to lower bound the
RHSs of \eqref{eq:rankprobZ} and \eqref{eq:rankprobY} by selecting $\epsilon =
1$, $\S$ such that its diagonal entries equal 1, and $\beta
=\frac{1}{\mu_0(\V_L)}$ for the RHS of \eqref{eq:rankprobZ} and $\beta =
\frac{1}{\mu_0(\U_L)}$ for the RHS of \eqref{eq:rankprobY}.
In particular, given the lower bounds on $d$ and $l$ in the statement of the
corollary, the RHSs are each lower bounded by $\sqrt{1-\delta}$.
Furthermore, by the independence of row and column sampling and
\eqref{eq:rankprobZ} and \eqref{eq:rankprobY}, we see that
\begin{align*}
	1-\delta & \leq \Parg{\rank{\U_{d}} =
r}\Parg{\rank{\V_{l}} = r} \\ 
&  = \Parg{\rank{\Y_1} = r}\Parg{\rank{\Z_1} = r} \\
& = \Parg{\rank{\W}=r}.
\end{align*}
  Finally, Prop.~\ref{prop:exact} implies that
  $$\Parg{\L = \Lgnys} \geq \Parg{\rank{\W}=r}\geq 1-\delta,$$
  which proves the statement of the theorem.

\section{Proof of Corollary~\lowercase{\ref{cor:rp-main}}: Random Projection}
\label{sec:rp-proof}
Our proof rests upon the following random projection guarantee of  \citet{HaMaTr11}:

\begin{theorem}[{\cite[Thm.~10.7]{HaMaTr11}}]
\label{thm:rp-orig}
Given a matrix $\M\in\reals^{m\times n}$ and a rank-$r$ approximation
$\L\in\reals^{m\times n}$ with $r \ge 2$, choose an oversampling parameter $p
\ge 4$, where $r + p \le \minarg{m,n}$.  Draw an $n \times (r+p)$ standard
Gaussian matrix $\G$, let $\Y = \M \G$. For
all $u,t \ge 1$, 
$$\norm{\M - \P_Y\M}_F \leq (1+t \sqrt{12r/p})\norm{\M -\M_r}_F + ut \cdot \frac{e \sqrt{r+p}}{p+1} \norm{\M - \M_r}$$ 
with
probability at least $1-5t^{-p} - 2e^{-u^2/2}$.
\end{theorem}

Fix $(u, t)  =(\sqrt{2\log(7/\delta)},e )$, and note that  
$$1-5e^{-p} - 2e^{-u^2/2} = 1-5e^{-p} - 2\delta/7 \geq 1-\delta,$$ 
since $p \geq \log(7/\delta)$.
Hence, Thm.~\ref{thm:rp-orig} implies that
\begin{align*}
\norm{\M - \P_Y\M}_F 
	&\leq (1+ e\sqrt{12r/p})\norm{\M -\M_r}_F + \frac{e^2 \sqrt{2(r+p)\log(7/\delta)}}{p+1} \norm{\M - \M_r}_2 \\
	&\leq \left(1+ e\sqrt{12r/p}+ \frac{e^2 \sqrt{2(r+p)\log(7/\delta)}}{p+1}\right) \norm{\M - \L}_F \\
	&\leq \left(1+ e\sqrt{12r/p}+ e^2 \sqrt{2r\log(7/\delta)/p}\right) \norm{\M - \L}_F \\
	&\leq \left(1+ 11 \sqrt{2r\log(7/\delta)/p}\right) \norm{\M - \L}_F 
	\leq (1+ \epsilon) \norm{\M - \L}_F 
\end{align*}
with probability at least  $1-\delta$,
where the second inequality follows from $\norm{\M - \M_r}_2 \leq \norm{\M - \M_r}_F \le \norm{\M - \L}_F$,
the third follows from $\sqrt{r+p}\sqrt{p} \leq (p+1)\sqrt{r}$ for all $r$ and $p$,
and the final follows from our choice of $p \ge 242\ r\log(7/\delta)/\epsilon^2$.

Next, we note, as in the proof of Thm.~9.3 of \citet{HaMaTr11}, that
$$\norm{\P_Y\M - \Lrp}_F \le \norm{\P_Y\M - \P_Y\M_r}_F \le \norm{\M - \M_r}_F
\le \norm{\M - \L}_F\,.$$  The first inequality holds because $\Lrp$ is by
definition the best rank-$r$ approximation to $\P_Y\M$ and $\rank{\P_Y\M_r} \le
r$. The second inequality holds since $ \norm{\M - \M_r}_F = \norm{\P_Y(\M -
\M_r)}_F + \norm{\P^\perp_Y(\M - \M_r)}_F$.  The final inequality holds since
$\M_r$ is the best rank-$r$ approximation to $\M$ and $\rank{\L} = r$.
Moroever, by the triangle inequality,
\begin{align}
\norm{\M - \Lrp}_F & \le \norm{\M - \P_Y \M}_F + \norm{\P_Y\M - \Lrp}_F \nonumber \\
\label{eqn:rp_rank_trunc}
& \le \norm{\M - \P_Y \M}_F + \norm{\M - \L}_F \,. 
\end{align}
Combining \eqref{eqn:rp_rank_trunc} with the first statement of the corollary
yields the second statement.

\section{Proof of Theorem~\lowercase{\ref{thm:master}}: Coherence Master Theorem}
\label{sec:proof-master}
\subsection{Proof of \projmf and \rpmf Bounds}
Let $\L_0 = [ \C_{0,1}, \ldots, \C_{0,t}]$ and $\tilde \L = [ \hat \C_{1},
\ldots, \hat \C_{t}]$.  Define  $A(\X)$ as the event that a matrix
$\X$ is $(\frac{r\mu^2}{1-\epsilon/2},r)$-coherent
and $K$ as the event $\norm{\tilde \L - \Lprojmf}_F
\le (1 + \epsilon) \norm{\L_0 - \tilde \L}_F$.  
When $K$ holds, we have that 
\begin{align*}
	\norm{\L_0 - \Lprojmf}_F &\le  \norm{\L_0 - \tilde \L}_F+\norm{\tilde \L - \Lprojmf}_F 
	\le (2 + \epsilon) \norm{\L_0 - \tilde \L}_F \\
	&= (2 + \epsilon) \sqrt{\textsum_{i=1}^t\norm{\C_{0,i} - \hat\C_i}_F^2},
\end{align*} by the triangle inequality, and hence
it suffices to lower bound $\Parg{K \cap \textbigcap_i A(\C_{0,i})}.$
Our choice of $l$, with a factor of $\log(2/\delta)$, implies that each
$A(\C_{0,i})$ holds with probability at least $1-\delta/(2n)$ by
Lem.~\ref{lem:sub-coh}, while $K$ holds with probability at least
$1-\delta/2$ by Cor.~\ref{cor:proj-main}.  Hence, by the union bound,
\begin{align*}
\Parg{K \cap \textbigcap_i A(\C_{0,i})} \geq 1-\Parg{K^c}-\textsum_{i}\Parg{A(\C_{0,i})^c} \geq 1-\delta/2 - t \delta/(2n)
\geq 1-\delta.
\end{align*}
An identical proof with Cor.~\ref{cor:rp-main} substituted for Cor.~\ref{cor:proj-main} yields the random projection result.

\subsection{Proof of \gnysmf Bound}
To prove the generalized \nys result, we redefine $\tilde \L$ and write it in block
notation as:
$$ \tilde \L = \begin{bmatrix} \hat \C_{1} & \hat \R_{2} \\
			        \hat \C_{2} & \L_{0,22} \end{bmatrix}\,, 
     \quad \text{where} \quad 
     \hat \C = \begin{bmatrix} \hat \C_{1} \\ \hat \C_{2} \end{bmatrix}
     \, , \quad 
     \hat \R = \begin{bmatrix} \hat \R_{1} & \hat \R_{2} \end{bmatrix} 
$$
and $\L_{0,22} \in \reals^{(m-d) \times (n-l)}$ is the bottom right submatrix
of $\L_0$.  
We further redefine $K$ as the event 
$\norm{\tilde \L - \Lgnysmf}_F \le (1 + \epsilon)^2 \norm{\L_0 - \tilde \L}_F$.  
As above,
$$\norm{\L_0 - \Lgnysmf}_F \le  \norm{\L_0 - \tilde \L}_F+\norm{\tilde \L - \Lgnysmf}_F 
\le (2+2\epsilon + \epsilon^2) \norm{\L_0 - \tilde \L}_F
\le (2 + 3\epsilon) \norm{\L_0 - \tilde \L}_F,$$
when $K$ holds, by the triangle inequality.  
Our choices of $l$ and 
$$d\geq cl\mu_0(\hat\C)\log(m)\log(4/\delta)/\epsilon^2 \geq cr\mu\log(m)\log(1/\delta)/\epsilon^2$$ 
imply that $A(\C)$ and $A(\R)$ hold with probability 
at least $1-\delta/(2n)$ and $1-\delta/(4n)$ respectively by Lem.~\ref{lem:sub-coh}, 
while $K$ holds with probability at least $(1-\delta/2)(1-\delta/4 - 0.2)$ by Cor.~\ref{cor:gnys-main}.
Hence, by the union bound,
\begin{align*}
\Parg{K \cap A(\C) \cap A(\R)} &\geq 1-\Parg{K^c}-\Parg{A(\C)^c}-\Parg{A(\R)^c} \\
&\geq 1 - (1 - (1-\delta/2)(1-\delta/4 -0.2)) - \delta/(2n) - \delta/(4n) \\
&\geq (1-\delta/2)(1-\delta/4 -0.2) - 3\delta/8 \\
&\geq (1-\delta)(1-\delta-0.2)
\end{align*}
for all $n \geq 2$ and $\delta \leq 0.8$.

\section{Proof of Corollary~\lowercase{\ref{cor:fast-mc-noise}}: \fastmf-MC under Incoherence}
\label{sec:fast-mc-noise}
\subsection{Proof of \projmf and \rpmf Bounds}
We begin by proving the \projmf bound.  
Let $G$ be the event that $$\norm{\L_0 - \Lprojmf}_F \leq (2+ \epsilon)c_e\sqrt{mn}\Delta,$$
$H$ be the event that $$\norm{\L_0 - \Lprojmf}_F \leq (2+ \epsilon)\sqrt{\textsum_{i=1}^t \norm{\C_{0,i} - \hat\C_i}_F^2},$$
$A(\X)$ be the event that a matrix $\X$ is $(\frac{r\mu^2}{1-\epsilon/2},r)$-coherent,
and, for each $i\in\{1,\ldots,t\}$, 
$B_i$ be the event that $\norm{\C_{0,i} - \hat \C_i}_F > c_e\sqrt{ml}\Delta$.

Note that, by assumption,
\begin{align*}
l &\geq \textstyle{c\mu^2 r^2(m+n)n\beta\log^2(m+n)}/(s\epsilon^2)
\geq cr\mu\log(n)2\beta\log(m+n)/\epsilon^2 \\
&\geq cr\mu\log(n)((2\beta-2)\log(\bar{n}) + \log(2))/\epsilon^2 
=cr\mu\log(n)\log(2\bar{n}^{2\beta-2})/\epsilon^2.
\end{align*}
Hence the Coherence Master Theorem (Thm.~\ref{thm:master}) guarantees that,
with probability at least $1 - \bar{n}^{2-2\beta}$, $H$ holds and
the event $A(\C_{0,i})$ holds for each $i$.
Since $G$ holds whenever $H$ holds and $B_i^c$ holds for each $i$, we have
\begin{align*}
\Parg{G} &\geq \Parg{H \cap \textbigcap_i B_i^c} \geq \Parg{H\cap\textbigcap_i A(\C_{0,i})\cap\textbigcap_i B_i^c} \\
&= \Parg{H\cap\textbigcap_i A(\C_{0,i})}\Parg{\textbigcap_i B_i^c \mid H\cap\textbigcap_i A(\C_{0,i})} \\
&= \Parg{H\cap\textbigcap_i A(\C_{0,i})}(1 - \Parg{\textbigcup_i B_i \mid H\cap\textbigcap_i A(\C_{0,i})}) \\
&\geq (1-\bar{n}^{2-2\beta}) (1 - \textsum_i \Parg{B_i \mid A(\C_{0,i})}) \\
&\geq 1 - \bar{n}^{2-2\beta} - \textsum_i \Parg{B_i \mid A(\C_{0,i})}.
\end{align*}
To prove our desired claim, it therefore suffices to show 
$$\Parg{B_i \mid A(\C_{0,i})} \leq 4 \log(\bar{n})\bar{n}^{2-2\beta} + \bar{n}^{-2\beta} \leq 5 \log(\bar{n})\bar{n}^{2-2\beta} $$ for each $i$.

For each $i$, let $D_i$ be the event that $s_i < 32 \mu'r(m+l)\beta'\log^2(m+l)$, 
where $s_i$ is the number of revealed entries in $\C_{0,i}$,
$$\mu' \defeq \frac{\mu^2 r}{1-\epsilon/2},\quad\quad
\text{and}\quad\quad \beta' \defeq \frac{\beta\log(\bar{n})}{\log(\maxarg{m,l})}.$$  
By Thm.~\ref{thm:convex-mc-noise} and our choice of $\beta'$,
\begin{align*}
\Parg{B_i \mid A(\C_{0,i})}  &\leq \Parg{B_i \mid A(\C_{0,i}), D_i^c} +
\Parg{D_i \mid A(\C_{0,i})} \\
&\leq 4 \log(\maxarg{m,l})\maxarg{m,l}^{2-2\beta'}+ \Parg{D_i}\\
&\leq 4 \log(\bar{n})\bar{n}^{2-2\beta}+ \Parg{D_i}.
\end{align*}
Further, since the support of $\S_0$ is uniformly distributed and of
cardinality $s$, the variable $s_i$ has a hypergeometric distribution with
$\mE(s_i) = \frac{sl}{n}$ and hence satisfies Hoeffding's inequality for the
hypergeometric distribution~\cite[Sec.~6]{Hoeffding63}: $$\Parg{s_i \leq
\mE(s_i) - st} \leq \exp{-2st^2}.$$ 
Since, by assumption,
$$s \geq \textstyle{c\mu^2 r^2(m+n)n\beta\log^2(m+n)}/(l\epsilon^2)
\geq 64 \mu' r(m+l)n\beta'\log^2(m+l)/l,$$
and
$$sl^2/n^2 \geq \textstyle{c\mu^2 r^2(m+n)l\beta\log^2(m+n)}/(n\epsilon^2)
\geq 4\log(\bar{n})\beta,$$
it follows that	
\begin{align*}
\Parg{D_i} &= \Parg{s_i < \mE(s_i) - s\left(\frac{l}{n} - \frac{32
\mu' r(m+l)\beta'\log^2(m+l)}{s} \right)}\\
&\leq \Parg{s_i < \mE(s_i) - s\left(\frac{l}{n} - \frac{l}{2n} \right)}
= \Parg{s_i < \mE(s_i) - s\frac{l}{2n}} \\
&\leq \exp{-\frac{sl^2}{2n^2}} \leq \exp{-2\log(\bar{n})\beta} =
\bar{n}^{-2\beta}.
\end{align*}
Hence, $\Parg{B_i \mid A(\C_{0,i})} \leq
4\log(\bar{n})\bar{n}^{2-2\beta} + \bar{n}^{-2\beta}$ for each $i$, and the \projmf result
follows.

Since, $p \geq 242\ r\log(14\bar{n}^{2\beta-2})/\epsilon^2$, the \rpmf bound follows in an identical manner
from the Coherence Master Theorem (Thm.~\ref{thm:master}).

\subsection{Proof of \gnysmf Bound}
For \gnysmf, let $B_C$ be the event that $\norm{\C_{0} - \hat \C}_F >
c_e\sqrt{ml}\Delta$ and $B_R$ be the event that $\norm{\R_{0} - \hat \R}_F >
c_e\sqrt{dn}\Delta$.  
The Coherence Master Theorem (Thm.~\ref{thm:master}) and our choice of 
$$d \geq \textstyle{cl\mu_0(\hat{\C})(2\beta-1)\log^2(4\bar{n})}\bar{n}/(n\epsilon^2)
\geq \textstyle{cl\mu_0(\hat{\C})\log(m)\log(4\bar{n}^{2\beta-2})}/\epsilon^2$$
guarantee that,
with probability at least $(1 - \bar{n}^{2-2\beta})(1 - \bar{n}^{2-2\beta} - 0.2) \geq 1 - 2\bar{n}^{2-2\beta} - 0.2$, 
$$\norm{\L_0 - \Lgnysmf}_F \leq (2+3\epsilon)\sqrt{\norm{\C_0-\hat\C}_F^2+\norm{\R_0-\hat\R}_F^2},$$ 
and both $A(\C)$ and $A(\R)$ hold.
Moreover, since 
\begin{align*}
d &\geq \textstyle{cl\mu_0(\hat{\C})(2\beta-1)\log^2(4\bar{n})}\bar{n}/(n\epsilon^2)
\geq \textstyle{c\mu^2 r^2(m+n)\bar{n}\beta\log^2(m+n)}/(s\epsilon^2),
\end{align*}
reasoning identical to the \projmf case yields $\Parg{B_C \mid
A(\C)} \leq  4 \log(\bar{n})\bar{n}^{2-2\beta} + \bar{n}^{-2\beta}$  and $\Parg{B_R \mid A(\R)}
\leq  4 \log(\bar{n})\bar{n}^{2-2\beta} + \bar{n}^{-2\beta}$, and the \gnysmf bound 
follows as above.

\section{Proof of Corollary~\lowercase{\ref{cor:fast-rpca-noise}}: \fastmf-RMF under Incoherence}
\label{sec:fast-rpca-noise}
\subsection{Proof of \projmf and \rpmf Bounds}
We begin by proving the \projmf bound.  
Let $G$ be the event that $$\norm{\L_0 - \Lprojmf}_F \leq (2+ \epsilon)c_e'\sqrt{mn}\Delta$$
for the constant $c_e'$ defined in Thm.~\ref{thm:rpca-noise},
$H$ be the event that $$\norm{\L_0 - \Lprojmf}_F \leq (2+ \epsilon)\sqrt{\textsum_{i=1}^t \norm{\C_{0,i} - \hat\C_i}_F^2},$$
$A(\X)$ be the event that a matrix $\X$ is $(\frac{r\mu^2}{1-\epsilon/2},r)$-coherent,
and, for each $i\in\{1,\ldots,t\}$, 
$B_i$ be the event that $\norm{\C_{0,i} - \hat \C_i}_F > c_e'\sqrt{ml}\Delta$.

We may take $\rho_r \leq 1$, and hence, by assumption,
\begin{align*}
l \geq cr^2\mu^2\beta\log^2(2\bar{n})/(\epsilon^2\rho_r) 
\geq cr\mu\log(n)\log(2\bar{n}^{\beta})/\epsilon^2.
\end{align*}
Hence the Coherence Master Theorem (Thm.~\ref{thm:master}) guarantees that,
with probability at least $1 - \bar{n}^{-\beta}$, $H$ holds and
the event $A(\C_{0,i})$ holds for each $i$.
Since $G$ holds whenever $H$ holds and $B_i^c$ holds for each $i$, we have
\begin{align*}
\Parg{G} &\geq \Parg{H \cap \textbigcap_i B_i^c} \geq \Parg{H\cap\textbigcap_i A(\C_{0,i})\cap\textbigcap_i B_i^c} \\
&= \Parg{H\cap\textbigcap_i A(\C_{0,i})}\Parg{\textbigcap_i B_i^c \mid H\cap\textbigcap_i A(\C_{0,i})} \\
&= \Parg{H\cap\textbigcap_i A(\C_{0,i})}(1 - \Parg{\textbigcup_i B_i \mid H\cap\textbigcap_i A(\C_{0,i})}) \\
&\geq (1-\bar{n}^{-\beta}) (1 - \textsum_i \Parg{B_i \mid A(\C_{0,i})}) \\
&\geq 1 - \bar{n}^{-\beta} - \textsum_i \Parg{B_i \mid A(\C_{0,i})}.
\end{align*}
To prove our desired claim, it therefore suffices to show 
$$\Parg{B_i \mid A(\C_{0,i})} \leq (c_p+1) \bar{n}^{-\beta}$$ for each $i$.

Define $\bar{m} \defeq \maxarg{m,l}$ and $\beta'' \defeq
\beta\log(\bar{n})/\log(\bar{m})\leq\beta'$.
By assumption, 
$$
r \leq \frac{\rho_r m}{2\mu^2r\log^2(\bar{n})}
\leq \frac{\rho_r m(1-\epsilon/2)}{\mu^2r\log^2(\bar{m})}
\quad\text{and}\quad
r \leq \frac{\rho_rl\epsilon^2}{c\mu^2r\beta\log^2(2\bar{n})}
\leq \frac{\rho_r l(1-\epsilon/2)}{\mu^2r\log^2(\bar{m})}.$$
Hence, by Thm.~\ref{thm:rpca-noise} and the definitions of $\beta'$ and $\beta''$,
\begin{align*}
\Parg{B_i \mid A(\C_{0,i})}  &\leq \Parg{B_i \mid A(\C_{0,i}), s_i \leq
(1-\rho_s\beta'') ml} +  \Parg{s_i > (1-\rho_s\beta'') ml \mid A(\C_{0,i})} \\
&\leq c_p\bar{m}^{-\beta''} + \Parg{s_i > (1-\rho_s\beta'') ml }\\
&\leq c_p\bar{n}^{-\beta} + \Parg{s_i > (1-\rho_s\beta') ml },
\end{align*}
where $s_i$ is the number of corrupted entries in $\C_{0,i}$.
Further, since
the support of $\S_0$ is uniformly distributed and of cardinality $s$, the
variable $s_i$ has a hypergeometric distribution with $\mE(s_i) = \frac{sl}{n}$
and hence satisfies Bernstein's inequality for the
hypergeometric~\cite[Sec.~6]{Hoeffding63}: $$\Parg{s_i \geq \mE(s_i) + st} \leq
\exp{-st^2/(2\sigma^2+2t/3)} \leq \exp{-st^2n/4l},$$ for all $0\leq t \leq
3l/n$ and $\sigma^2 \defeq \frac{l}{n}(1-\frac{l}{n}) \leq \frac{l}{n}$.  It
therefore follows that	
\begin{align*}
\Parg{s_i > (1-\rho_s\beta') ml } &= \Parg{s_i > \mE(s_i) +
s\left(\frac{(1-\rho_s\beta') ml}{s} -\frac{l}{n}\right)}\\
&= \Parg{s_i > \mE(s_i) +
s\frac{l}{n}\left(\frac{(1-\rho_s\beta')}{(1-\rho_s\beta_s)} -1\right)}\\
&\leq \exp{-s\frac{l}{4n}\left(\frac{(1-\rho_s\beta')}{(1-\rho_s\beta_s)}
-1\right)^2}\\	
&= \exp{-\frac{ml}{4}\frac{(\rho_s\beta_s -\rho_s\beta')^2}{(1-\rho_s\beta_s)}}
\leq \bar{n}^{-\beta}
\end{align*}
by our assumptions on $s$ and $l$ and the fact that
$\frac{l}{n}\left(\frac{(1-\rho_s\beta')}{(1-\rho_s\beta_s)} -1\right) \leq
3l/n$ whenever $4\beta_s-3/\rho_s\leq \beta'$.  Hence, $\Parg{B_i \mid
A(\C_{0,i})} \leq (c_p+1)\bar{n}^{-\beta}$ for each $i$, and the \projmf result
follows.

Since, $p \geq 242\ r\log(14\bar{n}^{\beta})/\epsilon^2$, the \rpmf bound follows in an identical manner
from the Coherence Master Theorem (Thm.~\ref{thm:master}).

\subsection{Proof of \gnysmf Bound}
For \gnysmf, let $B_C$ be the event that $\norm{\C_{0} - \hat \C}_F >
c_e'\sqrt{ml}\Delta$ and $B_R$ be the event that $\norm{\R_{0} - \hat \R}_F >
c_e'\sqrt{dn}\Delta$.  
The Coherence Master Theorem (Thm.~\ref{thm:master}) and our choice of 
$d\geq cl\mu_0(\hat\C)\beta\log^2(4\bar{n})/\epsilon^2$ guarantee that,
with probability at least $(1 - \bar{n}^{-\beta})(1 - \bar{n}^{-\beta} - 0.2) \geq 1-2\bar{n}^{-\beta} - 0.2$, 
$$\norm{\L_0 - \Lgnysmf}_F \leq (2+3\epsilon)\sqrt{\norm{\C_0-\hat\C}_F^2+\norm{\R_0-\hat\R}_F^2},$$ 
and both $A(\C)$ and $A(\R)$ hold.
Moreover, since 
\begin{align*}
d &\geq cl\mu_0(\hat\C)\beta\log^2(4\bar{n})/\epsilon^2
\geq \textstyle{c\mu^2 r^2\beta\log^2(\bar{n})}/(\epsilon^2\rho_r),
\end{align*}
reasoning identical to the \projmf case yields $\Parg{B_C \mid
A(\C)} \leq  (c_p+1)\bar{n}^{-\beta}$  and $\Parg{B_R \mid A(\R)}
\leq (c_p+1)\bar{n}^{-\beta}$, and the \gnysmf bound 
follows as above.

\section{Proof of Theorem~\lowercase{\ref{thm:convex-mc-noise}}: Noisy MC under Incoherence}
\label{sec:convex-mc-noise}
In the spirit of \citet{CandesPl10}, our proof will extend the noiseless
analysis of \citet{Recht11} to the noisy matrix completion setting.  As
suggested in \citet{GrossNe10}, we will obtain strengthened results, even in the
noiseless case, by reasoning directly about the without-replacement sampling
model, rather than appealing to a with-replacement surrogate, as done in
\citet{Recht11}.

For $\U_{L_0}\mSigma_{L_0}\V_{L_0}^\top$ the compact SVD of $\L_0$, we let $T =
\{\U_{L_0}\X + \Y\V_{L_0}^\top : \X\in\reals^{r\times n},\Y \in\reals^{m\times
r} \}$, $\proj_T$ denote orthogonal projection onto the space $T$, and
$\proj_{T^\bot}$ represent orthogonal projection onto the orthogonal complement
of $T$.  We further define $\mathcal{I}$ as the identity operator on
$\reals^{m\times n}$ and the spectral norm of an operator $\mc{A}:
\reals^{m\times n} \rightarrow \reals^{m\times n}$ as $\norm{\mc{A}}_2 =
\sup_{\norm{\X}_F\leq 1} \norm{\mc{A}(\X)}_F$.

We begin with a theorem providing sufficient conditions for our desired
estimation guarantee.
\begin{theorem}
\label{thm:suff-mc-noise}
Under the assumptions of Thm.~\ref{thm:convex-mc-noise}, suppose that
\begin{align}
\label{eq:A1}
\displaystyle\frac{mn}{s}\flexnorm{\proj_T\obsproj\proj_T -
\frac{s}{mn}\proj_T}_2\leq \half
\end{align}
and that there exists a $\Y=\obsproj(\Y)\in\reals^{m\times n}$ satisfying
\begin{align}
\label{eq:A2}
\norm{\proj_T(\Y) - \U_{L_0}\V_{L_0}^\top}_F \leq\sqrt{\frac{s}{32mn}}\quad
\text{and}\quad \norm{\proj_{T^\bot}(\Y)}_2 < \half.
\end{align}
Then, $$\norm{\L_0 - \hat{\L}}_F \leq 8\sqrt{\frac{2m^2n}{s}+m+\frac{1}{16}
}\Delta \leq c_e''\sqrt{mn}\Delta.$$
\end{theorem}
\begin{proof}
We may write $\hat{\L}$ as $\L_0 + \G + \H$, where $\obsproj(\G) = \G$ and
$\obsproj(\H) = \bf{0}$.  Then, under \eqref{eq:A1},
$$\norm{\obsproj\proj_{T}(\H)}_F^2 = \<\H,\proj_T\obsproj^2\proj_T(\H)\> \geq
\<\H,\proj_T\obsproj\proj_T(\H)\> \geq \frac{s}{2mn}\norm{\proj_T(\H)}^2_F.$$
Furthermore, by the triangle inequality, $0 = \norm{\obsproj(\H)}_F \geq
\norm{\obsproj\proj_T(\H)}_F - \norm{\obsproj\proj_{T^\bot}(\H)}_F.$ Hence, we
have
\begin{align}
\label{eq:bound-fro-T}
\sqrt{\frac{s}{2mn}}\norm{\proj_T(\H)}_F\leq \norm{\obsproj\proj_{T}(\H)}_F
\leq\norm{\obsproj\proj_{T^\bot}(\H)}_F \leq \norm{\proj_{T^\bot}(\H)}_F \leq
\norm{\proj_{T^\bot}(\H)}_*,
\end{align}
where the penultimate inequality follows as $\obsproj$ is an orthogonal
projection operator.
	 
Next we select $\U_{\bot}$ and $\V_{\bot}$ such that $[\U_{L_0},\U_{\bot}]$ and
$[\V_{L_0},\V_{\bot}]$ are orthonormal and
$\<\U_{\bot}\V_{\bot}^\top,\proj_{T^{\bot}}(\H)\>=\norm{\proj_{T^\bot}(\H)}_*$
and note that
\begin{align*}
\norm{\L_0+\H}_* &\geq \<\U_{L_0}\V_{L_0}^\top + \U_\bot\V_\bot^\top, \L_0+\H\> \\
&= \norm{\L_0}_* + \<\U_{L_0}\V_{L_0}^\top + \U_\bot\V_\bot^\top - \Y, \H\> \\
&= \norm{\L_0}_* + \<\U_{L_0}\V_{L_0}^\top - \proj_{T}(\Y), \proj_{T}(\H)\> +
\<\U_\bot\V_\bot^\top,\proj_{T^\bot}(\H)\>  - \<\proj_{T^\bot}(\Y),
\proj_{T^\bot}(\H)\> \\
&\geq \norm{\L_0}_* -\norm{\U_{L_0}\V_{L_0}^\top - \proj_{T}(\Y)}_F
\norm{\proj_{T}(\H)}_F+\norm{\proj_{T^\bot}(\H)}_* -
\norm{\proj_{T^\bot}(\Y)}_2\norm{\proj_{T^\bot}(\H)}_*\\
&> \norm{\L_0}_* + \half\norm{\proj_{T^\bot}(\H)}_*-
\sqrt{\frac{s}{32mn}}\norm{\proj_{T}(\H)}_F\\
&\geq \norm{\L_0}_* + \frac{1}{4}\norm{\proj_{T^\bot}(\H)}_F
\end{align*}  
where the first inequality follows from the variational representation of the
trace norm, $\norm{\A}_* = \sup_{\norm{\B}_2\leq1}\<\A,\B\>$, the first
equality follows from the fact that $\<\Y,\H\> = 0$ for $\Y= \obsproj(\Y)$, the
second inequality follows from H\"older's inequality for Schatten $p$-norms,
the third inequality follows from \eqref{eq:A2}, and the final inequality
follows from \eqref{eq:bound-fro-T}.
	 
Since $\L_0$ is feasible for \eqref{eqn:convex-mc-noise},
$\norm{\L_0}_*\geq\norm{\hat{\L}}_*$, and, by the triangle inequality,
$\norm{\hat{\L}}_* \geq \norm{\L_0+\H}_* - \norm{\G}_*$.  Since $\norm{\G}_*
\leq \sqrt{m}\norm{\G}_F$ and $\norm{\G}_F \leq
\norm{\obsproj(\hat{\L}-\M)}_F+\norm{\obsproj(\M-\L_0)}_F\leq 2\Delta$, we
conclude that
\begin{align*}
\norm{\L_0 - \hat{\L}}_F^2 &=
\norm{\proj_{T}(\H)}_F^2+\norm{\proj_{T^\bot}(\H)}_F^2 + \norm{\G}_F^2 \\
&\leq \left(\frac{2mn}{s}+1\right)\norm{\proj_{T^\bot}(\H)}_F^2 + \norm{\G}_F^2 \\
&\leq 16\left(\frac{2mn}{s}+1\right)\norm{\G}_*^2 + \norm{\G}_F^2\\
&\leq 64\left(\frac{2m^2n}{s}+m+\frac{1}{16}\right) \Delta^2.
\end{align*}
Hence $$\norm{\L_0 - \hat{\L}}_F\leq 8\sqrt{\frac{2m^2n}{s}+m+\frac{1}{16}
}\Delta \leq c_e''\sqrt{mn}\Delta$$
for some constant $c_e''$, by our assumption on $s$.
\end{proof}

To show that the sufficient conditions of Thm.~\ref{thm:suff-mc-noise} hold
with high probability,  we will require four lemmas.  The first establishes
that the operator $\proj_T\obsproj\proj_T$ is nearly an isometry on $T$ when
sufficiently many entries are sampled.
\begin{lemma}
\label{lem:mc-isometry}
For all $\beta > 1,$ $$\frac{mn}{s}\flexnorm{\proj_T\obsproj\proj_T -
\frac{s}{mn}\proj_T}_2 \leq \sqrt{\frac{16\mu r(m+n)\beta\log(n)}{3s}}$$ with
probability at least $1-2n^{2-2\beta}$ provided that $s > \frac{16}{3}\mu r
(n+m) \beta \log(n)$.
\end{lemma}

The second states that a sparsely but uniformly observed matrix is close to a
multiple of the original matrix under the spectral norm.
\begin{lemma}
\label{lem:mc-spec}
Let $\Z$ be a fixed matrix in $\reals^{m\times n}$.  Then for all $\beta > 1,$
$$\flexnorm{\left(\frac{mn}{s}\obsproj - \mc{I}\right)(\Z)}_2 \leq
\sqrt{\frac{8\beta mn^2\log(m+n)}{3s}}\norm{\Z}_\infty$$ with probability at
least $1 - (m+n)^{1-\beta}$ provided that $s > 6\beta m \log(m + n).$
\end{lemma}

The third asserts that the matrix infinity norm of a matrix in $T$ does not
increase under the operator $\proj_T\obsproj$.
\begin{lemma}
\label{lem:mc-infty}
Let $\Z \in T$ be a fixed matrix.  Then for all $\beta > 2$
$$\flexnorm{\frac{mn}{s}\proj_T\obsproj(\Z) - \Z}_\infty \leq
\sqrt{\frac{8\beta\mu r (m+n)\log(n)}{3s}}\norm{\Z}_\infty$$ with probability
at least $1-2n^{2-\beta}$ provided that $s> \frac{8}{3}\beta\mu r(m
+n)\log(n).$
\end{lemma}
These three lemmas were proved in \citet[Thm.~6, Thm.~7, and
Lem.~8]{Recht11} under the assumption that entry locations in $\obsset$ were
sampled \emph{with} replacement.  They admit identical proofs under the
sampling without replacement model by noting that the referenced Noncommutative
Bernstein Inequality \cite[Thm.~4]{Recht11} also holds under sampling without
replacement, as shown in \citet{GrossNe10}. 

Lem.~\ref{lem:mc-isometry} guarantees that \eqref{eq:A1} holds with high
probability.  To construct a matrix $\Y = \obsproj(\Y)$ satisfying
\eqref{eq:A2}, we consider a sampling with batch replacement scheme recommended
in \citet{GrossNe10} and developed in \citet{ChenXuCaSa11}.  Let $\batchset_1,
\ldots, \batchset_p$ be independent sets, each consisting of $q$ random entry
locations sampled without replacement, where $pq=s$.  Let $\batchset =
\cup_{i=1}^p \batchset_i$, and note that there exist $p$ and $q$ satisfying
$$q\geq \frac{128}{3}\mu r(m+n)\beta\log(m+n)\quad\text{and}\quad
p\geq\frac{3}{4}\log(n/2).$$ It suffices to establish \eqref{eq:A2} under this
batch replacement scheme, as shown in the next lemma.

\begin{lemma}
For any location set $\obsset_0\subset\{1,\ldots,m\}\times\{1,\ldots,n\}$, let
$A(\obsset_0)$ be the event that there exists
$\Y=\proj_{\obsset_0}(\Y)\in\reals^{m\times n}$ satisfying \eqref{eq:A2}.  If
$\obsset(s)$ consists of $s$ locations sampled uniformly without replacement
and $\batchset(s)$ is sampled via batch replacement with $p$ batches of size
$q$ for $pq=s$, then $\P(A(\batchset(s))) \leq \Parg{A(\obsset(s))}$.
\end{lemma}
\begin{proof}
As sketched in \citet{GrossNe10}
\begin{align*}
\Parg{A(\tilde{\obsset(s)})} &= \sum_{i=1}^s\P(|\batchset|=i)\P(A(\batchset(i))
\mid |\batchset|=i)\\
&\leq \sum_{i=1}^s\P(|\batchset|=i)\Parg{A(\obsset(i))}\\ 
&\leq \sum_{i=1}^s\P(|\batchset|=i)\Parg{A(\obsset(s))} = \Parg{A(\obsset(s))},
\end{align*}
since the probability of existence never decreases with more entries sampled
without replacement and, given the size of $\batchset$, the locations of
$\batchset$ are conditionally distributed uniformly (without replacement).
\end{proof}

We now follow the construction of \citet{Recht11} to obtain $\Y =
\batchproj{}(\Y)$ satisfying \eqref{eq:A2}.  Let $\W_0 = \U_{L_0}\V_{L_0}^\top$
and define $\Y_k = \frac{mn}{q}\sum_{j=1}^k\batchproj{j}(\W_{j-1})$ and
$\W_k=\U_{L_0}\V_{L_0}^\top - \proj_T(\Y_k)$ for $k = 1,\ldots,p$.  Assume that 
\begin{align}
\label{eq:B1}
\frac{mn}{q}\flexnorm{\proj_T\batchproj{k}\proj_T - \frac{q}{mn}\proj_T}_2 \leq \half
\end{align}
for all $k$.  Then $$\norm{\W_k}_F =
\flexnorm{\W_{k-1}-\frac{mn}{q}\proj_T\batchproj{k}(\W_{k-1})}_F =
\flexnorm{(\proj_T - \frac{mn}{q}\proj_T\batchproj{k}\proj_T)(\W_{k-1})}_F \leq
\half \norm{\W_{k-1}}_F$$ and hence $\norm{\W_k}_F\leq 2^{-k}\norm{\W_0}_F=
2^{-k}\sqrt{r}.$ Since $p \geq \frac{3}{4}\log(n/2) \geq \half \log_2(n/2) \geq
\log_2\sqrt{32rmn/s}$, $\Y \defeq \Y_p$ satisfies the first condition of
\eqref{eq:A2}.  

The second condition of \eqref{eq:A2} follows from the assumptions
\begin{align}
\label{eq:B2}
\flexnorm{\W_{k-1} - \frac{mn}{q}\proj_T\batchproj{k}(\W_{k-1})}_\infty &\leq
\half\norm{\W_{k-1}}_\infty\\
\label{eq:B3}
\flexnorm{\left(\frac{mn}{q}\batchproj{k} -\mc{I}\right)(\W_{k-1})}_2 &\leq
\sqrt{\frac{8mn^2\beta\log(m+n)}{3q}}\norm{\W_{k-1}}_\infty
\end{align}
for all $k$, since \eqref{eq:B2} implies $\norm{\W_k}_\infty \leq
2^{-k}\norm{\U_{L_0}\V_{L_0}^\top}_\infty$, and thus
\begin{align*}
\norm{\proj_{T^\bot}(\Y_p)}_2 &\leq
\sum_{j=1}^p\flexnorm{\frac{mn}{q}\proj_{T^\bot}\batchproj{j}(\W_{j-1})}_2
= \sum_{j=1}^p\flexnorm{\proj_{T^\bot}(\frac{mn}{q}\batchproj{j}(\W_{j-1}) - \W_{j-1})}_2\\
&\leq \sum_{j=1}^p\flexnorm{(\frac{mn}{q}\batchproj{j} - \mc{I})(\W_{j-1})}_2\\
&\leq \sum_{j=1}^p\sqrt{\frac{8mn^2\beta\log(m+n)}{3q}}\norm{\W_{j-1}}_\infty\\
&= 2\sum_{j=1}^p2^{-j}\sqrt{\frac{8mn^2\beta\log(m+n)}{3q}}
\norm{\U_W\V_W^\top}_\infty <\sqrt{\frac{32\mu rn\beta\log(m+n)}{3q}} < 1/2
\end{align*}
by our assumption on $q$.  The first line applies the triangle inequality; the
second holds since $\W_{j-1}\in T$ for each $j$; the third follows because
$\proj_{T^\bot}$ is an orthogonal projection; and the final line exploits
$(\mu,r)$-coherence.

We conclude by bounding the probability of any assumed event failing.
Lem.~\ref{lem:mc-isometry} implies that \eqref{eq:A1} fails to hold with
probability at most $2n^{2-2\beta}$.  For each $k$, \eqref{eq:B1} fails to hold
with probability at most $2n^{2-2\beta}$ by Lem.~\ref{lem:mc-isometry},
\eqref{eq:B2} fails to hold with probability at most $2n^{2-2\beta}$ by
Lem.~\ref{lem:mc-infty},  and \eqref{eq:B3} fails to hold with probability at
most $(m+n)^{1-2\beta}$ by Lem.~\ref{lem:mc-spec}.  Hence, by the union bound,
the conclusion of Thm.~\ref{thm:suff-mc-noise} holds with probability at least
$$1-2n^{2-2\beta} - \frac{3}{4}\log(n/2)(4n^{2-2\beta}+(m+n)^{1-2\beta})\geq
1-\frac{15}{4}\log(n)n^{2-2\beta} \geq 1-4\log(n)n^{2-2\beta}.$$

\section{Proof of Lemma~\lowercase{\ref{lem:sub-spike}}: Conservation of Non-Spikiness}
\label{sec:sub-spike}
By assumption, $$\L_C\L_C^\top = \sum_{a=1}^l \L^{(j_a)}(\L^{(j_a)})^\top$$
where $\{j_1,\dots, j_l\}$ are random indices drawn uniformly and without replacement from $\{1,\dots,n\}$.
Hence, we have that
\begin{align*}
\Earg{\norm{\L_C}_F^2} &= \Earg{\Trarg{\L_C\L_C^\top}} = \Trarg{\Earg{ \sum_{a=1}^l \L^{(j_a)}(\L^{(j_a)})^\top} } \\
&= \Trarg{  \sum_{a=1}^l \frac{1}{n} \sum_{j=1}^n  \L^{(j)}(\L^{(j)})^\top } = \frac{l}{n} \Trarg{\L\L^\top} = \frac{l}{n}\norm{\L}_F^2.
\end{align*}

Since $\norm{\L^{(j)}}^4 \leq m^2\norm{\L}_\infty^4$ for all $j\in\{1,\dots,n\}$, 
Hoeffding's inequality for sampling without replacement~\cite[Sec.~6]{Hoeffding63} implies
\begin{align*}
	\Parg{(1-\epsilon)(l/n)\norm{\L}_F^2 \geq \norm{\L_C}_F^2 } &\leq \exp{-2\epsilon^2\norm{\L}_F^4l^2/(n^2lm^2\norm{\L}_\infty^4)} \\
	&= \exp{-2\epsilon^2l/\spikiness^4(\L)} \leq \delta,
\end{align*}
by our choice of $l$.
Hence, $$\sqrt{l}\frac{1}{\norm{\L_C}_F} \leq \frac{\sqrt{n}}{\sqrt{1-\epsilon}}\frac{1}{\norm{\L}_F}$$ with probability at least $1-\delta$.
Since, $\norm{\L_C}_\infty \leq \norm{\L}_\infty$ almost surely, we have that
$$\spikiness(\L_C) = \frac{\sqrt{ml}\norm{\L_C}_\infty}{\norm{\L_C}_F} \leq  \frac{\sqrt{mn}\norm{\L}_\infty}{\sqrt{1-\epsilon}\norm{\L}_F} = \frac{\spikiness(\L)}{\sqrt{1-\epsilon}}$$
with probability at least $1-\delta$ as desired.

\section{Proof of Theorem~\lowercase{\ref{cor:proj-main-spike}}: Column Projection under Non-Spikiness}
\label{sec:proj-spike}
We now give a proof of Thm.~\ref{cor:proj-main-spike}.  While the results of this
section are stated in terms of i.i.d.\ with-replacement sampling of
columns and rows, a simple argument due to~\cite[Sec.~6]{Hoeffding63} implies
the same conclusions when columns and rows are sampled without replacement.

Our proof builds upon two key results from the randomized matrix approximation literature.
The first relates column projection to randomized matrix multiplication:
\begin{theorem}[Thm.~2 of \citep{DrineasKaMa06b}]
	\label{thm:col-proj-additive}
	Let $\G\in\reals^{m\times l}$ be a matrix of $l$ columns of $\A\in\reals^{m\times n}$, and let $r$ be a nonnegative integer.
	Then,
	$$\norm{\A - \G_r\G_r^+\A}_F \leq \norm{\A - \A_r}_F+ \sqrt{r} \norm{\A\A^\top - (n/l)\G\G^\top}_F.$$
\end{theorem}
The second allows us to bound $\norm{\A\A^\top - (n/l)\G\G^\top}_F$ in probability when entries are bounded:
\begin{lemma}[Lem.~2 of \citep{DrineasKaMa06a}]
\label{lem:mat-mult-bounded}
Given a failure probability $\delta\in(0,1]$ and matrices $\A\in\reals^{m\times k}$ and $ \B\in\reals^{k\times n}$ 
with $\norm{\A}_\infty \leq b$ and $\norm{\B}_\infty \leq b$, 
suppose that $\G$ is a matrix of $l$ columns drawn uniformly with replacement
from $\A$ and that $\H$ is a matrix of the corresponding $l$ rows of $\B$.
Then, with probability at least $1-\delta$, 
\begin{equation*}
|(\A\B)_{ij} - (n/l)(\G\H)_{ij}| \leq \frac{kb^2}{\sqrt{l}}\sqrt{8\log(2mn/\delta)}\quad \forall i,j.
\end{equation*}
\end{lemma}

Under our assumption, $\norm{\M}_\infty$ is bounded by $\alpha/\sqrt{mn}$. 
Hence, Lem.~\ref{lem:mat-mult-bounded} with  $\A=\M$ and $\B = \M^\top$ guarantees
$$\norm{\M\M^\top - (n/l)\C\C^\top}_F^2 \leq \frac{m^2n^2\spikiness^48\log(2mn/\delta)}{m^2n^2l} \leq \epsilon^2/r $$
with probability at least $1-\delta$, by our choice of $l$.

Now, Thm.~\ref{thm:col-proj-additive} implies that
\begin{align*}
	\norm{\M - \C\C^+\M}_F \leq \norm{\M - \C_r\C_r^+\M}_F &\leq \norm{\M - \M_r}_F+ \sqrt{r} \norm{\M\M^\top - (n/l)\C\C^\top}_F\\
	&\leq \norm{\M - \L}_F + \epsilon
\end{align*}
with probability at least $1-\delta$, as desired.

\section{Proof of Theorem~\lowercase{\ref{thm:master-spike}}: Spikiness Master Theorem}
\label{sec:proof-master_spike}
Define $A(\X)$ as the event that a matrix
$\X$ is $(\spikiness\sqrt{1+\epsilon/(4\sqrt{r})})$-spiky.
Since $\sqrt{1+\epsilon/(4\sqrt{r})} \leq \sqrt{1.25}$ for all $\epsilon\in (0,1]$ and $r\geq 1$,
$\X$ is $(\sqrt{1.25}\spikiness)$-spiky whenever $A(\X)$ holds.

Let $\L_0 = [ \C_{0,1}, \ldots, \C_{0,t}]$ and $\tilde \L = [ \hat \C_{1},
\ldots, \hat \C_{t}]$, and define $H$ as the event $\norm{\tilde \L - \Lprojmf}_F
\le \norm{\L_0 - \tilde \L}_F + \epsilon$.  
When $H$ holds, we have that 
\begin{align*}
	\norm{\L_0 - \Lprojmf}_F &\le  \norm{\L_0 - \tilde \L}_F+\norm{\tilde \L - \Lprojmf}_F 
	\le 2\norm{\L_0 - \tilde \L}_F + \epsilon\\
	&= 2\sqrt{\textsum_{i=1}^t\norm{\C_{0,i} - \hat\C_i}_F^2} + \epsilon,
\end{align*} by the triangle inequality, and hence
it suffices to lower bound $\Parg{H \cap \textbigcap_i A(\C_{0,i})}.$

By assumption,
\begin{align*}
l\geq 13r\spikiness^4\log(4mn/\delta)/\epsilon^2 
  \geq \spikiness^4\log(2n/\delta)/(2\tilde\epsilon^2)
\end{align*}
where $\tilde\epsilon \defeq \epsilon/(5\sqrt{r})$. %
Hence, for each $i$, Lem.~\ref{lem:sub-spike} implies that $\spikiness(\C_{0,i})\leq 
\spikiness/\sqrt{1-\tilde\epsilon}$ with probability at least $1-\delta/(2n)$.
Since 
$$(1-\epsilon/(5\sqrt{r}))(1+\epsilon/(4\sqrt{r})) = 1 + {\epsilon}{}(1-\epsilon/\sqrt{r})/(20\sqrt{r}) \geq 1$$
it follows that
$$\frac{1}{\sqrt{1-\tilde\epsilon}} = \frac{1}{\sqrt{1-\epsilon/(5\sqrt{r})}} \leq \sqrt{1+\epsilon/(4\sqrt{r})},$$ 
so that each event $A(\C_{0,i})$ also holds with probability at least $1-\delta/(2n)$.

Our assumption that $\norm{\hat\C_i}_\infty \leq \sqrt{1.25}\spikiness/\sqrt{mn}$ for all $i$
implies that $\norm{\tilde \L}_\infty \leq \sqrt{1.25}\spikiness/\sqrt{mn}$.
Our choice of $l$, with a factor of $\log(4mn/\delta)$, therefore implies that 
$H$ holds with probability at least
$1-\delta/2$ by Thm.~\ref{cor:proj-main-spike}.  Hence, by the union bound,
\begin{align*}
\Parg{H \cap \textbigcap_i A(\C_{0,i})} \geq 1-\Parg{H^c}-\textsum_{i}\Parg{A(\C_{0,i})^c} \geq 1-\delta/2 - t \delta/(2n)
\geq 1-\delta.
\end{align*}

To establish the \rpmf bound, redefine $H$ as the event $\norm{\tilde \L - \Lrp}_F \le (2+\epsilon)\norm{\L_0 - \tilde \L}_F$.
Since $p \geq 242\ r\log(14/\delta)/\epsilon^2$, $H$ holds with probability at least
$1-\delta/2$ by Cor.~\ref{cor:rp-main}, and the \rpmf bound follows as above.

\section{Proof of Corollary~\lowercase{\ref{cor:fast-mc-spike}}: Noisy MC under Non-Spikiness}
\label{sec:fast-mc-spike}
\subsection{Proof of \projmf Bound}
We begin by proving the \projmf bound.  
Let $G$ be the event that $$\norm{\L_0 - \Lprojmf}_F \leq 2\sqrt{c_1\maxarg{({l/}{n})\nu^2, 1}/\beta} + \epsilon,$$
$H$ be the event that $$\norm{\L_0 - \Lprojmf}_F \leq 2\sqrt{\textsum_{i=1}^t \norm{\C_{0,i} - \hat\C_i}_F^2} + \epsilon,$$ 
$A(\X)$ be the event that a matrix $\X$ is $(\sqrt{1.25}\spikiness)$-spiky,
and, for each $i\in\{1,\ldots,t\}$, 
$B_i$ be the event that $\norm{\C_{0,i} - \hat \C_i}_F^2 > (l/n)c_1\maxarg{({l/}{n})\nu^2, 1}/\beta$.

By definition, $\norm{\hat\C_i}_\infty \leq \sqrt{1.25}\alpha/\sqrt{ml}$ for all $i$.
Furthermore, we have assumed that
\begin{align*}
l &\geq 13(c_3+1)\sqrt{\frac{(m+n)\log(m+n)\beta}{s}}nr\spikiness^4\log(4mn)/\epsilon^2 \\
&\geq13r\spikiness^4(\log(4mn)+c_3\log(m+n))/\epsilon^2
\geq13r\spikiness^4\log(4mn(m+l)^{c_3})/\epsilon^2.
\end{align*}
Hence the Spikiness Master Theorem (Thm.~\ref{thm:master-spike}) guarantees that,
with probability at least $1 - \exp{-c_3\log(m + l)}$, $H$ holds and
the event $A(\C_{0,i})$ holds for each $i$.
Since $G$ holds whenever $H$ holds and $B_i^c$ holds for each $i$, we have
\begin{align*}
\Parg{G} &\geq \Parg{H \cap \textbigcap_i B_i^c} \geq \Parg{H\cap\textbigcap_i A(\C_{0,i})\cap\textbigcap_i B_i^c} \\
&= \Parg{H\cap\textbigcap_i A(\C_{0,i})}\Parg{\textbigcap_i B_i^c \mid H\cap\textbigcap_i A(\C_{0,i})} \\
&= \Parg{H\cap\textbigcap_i A(\C_{0,i})}(1 - \Parg{\textbigcup_i B_i \mid H\cap\textbigcap_i A(\C_{0,i})}) \\
&\geq (1- \exp{-c_3\log(m + l)}) (1 - \textsum_i \Parg{B_i \mid A(\C_{0,i})}) \\
&\geq 1 - (c_2+1)\exp{-c_3\log(m + l)} - \textsum_i \Parg{B_i \mid A(\C_{0,i})}.
\end{align*}
To prove our desired claim, it therefore suffices to show 
$$\Parg{B_i \mid A(\C_{0,i})} \leq (c_2+1) \exp{-c_3\log(m + l)} $$ for each $i$.

For each $i$, let $D_i$ be the event that $s_i < 1.25\spikiness^2 \beta (n/l)r(m+l)\log(m+l),$ 
where $s_i$ is the number of revealed entries in $\C_{0,i}$.
Since $\rank{\C_{0,i}} \leq \rank{\L_0} = r$ and
$\norm{\C_{0,i}}_F \leq \norm{\L_0}_F \leq 1$,
Cor.~\ref{cor:mc-spike} implies that
\begin{align} \label{eq:submc-bound-spike}
\Parg{B_i \mid A(\C_{0,i})}  &\leq \Parg{B_i \mid A(\C_{0,i}), D_i^c} +
\Parg{D_i \mid A(\C_{0,i})} \notag \\ 
&\leq c_2\exp{-c_3\log(m+l)} + \Parg{D_i}.
\end{align}

Further, since the support of $\S_0$ is uniformly distributed and of
cardinality $s$, the variable $s_i$ has a hypergeometric distribution with
$\mE(s_i) = \frac{sl}{n}$ and hence satisfies Hoeffding's inequality for the
hypergeometric distribution~\cite[Sec.~6]{Hoeffding63}: $$\Parg{s_i \leq
\mE(s_i) - st} \leq \exp{-2st^2}.$$ 
Our assumption on $l$ implies that
\begin{align*}
\frac{l}{n} 
&\geq 169(c_3+1)^2\spikiness^8 \beta\frac{n}{ls}r^2(m+n)\log(m+n)\log^2(4mn)/\epsilon^4 \\
&\geq 1.25\spikiness^2 \beta\frac{n}{ls}r(m+l)\log(m+l) + \sqrt{c_3\log(m+l)/(2s)},
\end{align*}
and therefore
\begin{align*}
\Parg{D_i} 
&= \Parg{s_i < \mE(s_i) - s\left(\frac{l}{n} - 1.25\spikiness^2 \beta\frac{n}{ls}r(m+l)\log(m+l)\right)}\\
&= \Parg{s_i < \mE(s_i) - s\sqrt{c_3\log(m+l)/(2s)}}\\
&\leq \exp{-2sc_3\log(m+l)/(2s)} = \exp{-c_3\log(m + l)}.
\end{align*}
Combined with \eqref{eq:submc-bound-spike}, this yields 
$\Parg{B_i \mid A(\C_{0,i})} \leq (c_2+1)\exp{-c_3\log(m + l)}$ for each $i$, and the \projmf result
follows.

\subsection{Proof of \rpmf Bound}
Let $G$ be the event that $$\norm{\L_0 - \Lrpmf}_F \leq (2+\epsilon)\sqrt{c_1\maxarg{({l/}{n})\nu^2, 1}/\beta}$$ and
$H$ be the event that $$\norm{\L_0 - \Lrpmf}_F \leq (2+\epsilon)\sqrt{\textsum_{i=1}^t \norm{\C_{0,i} - \hat\C_i}_F^2}.$$ 
Since $p \geq 242\ r\log(14(m + l)^{c_3})/\epsilon^2$, the \rpmf bound follows in an identical manner from 
the Spikiness Master Theorem (Thm.~\ref{thm:master-spike}).

\ifdefined\jmlrformat\else
\subsubsection*{Acknowledgments}
Lester Mackey gratefully acknowledges the support of DARPA through the National
Defense Science and Engineering Graduate Fellowship Program.  Ameet Talwalkar
gratefully acknowledges support from NSF award No. 1122732.

\fi

\ifdefined\jmlrformat
	\bibliography{refs}
\else
	\bibliographystyle{plainnat}
	{\small{\bibliography{refs}}}
\fi

\end{document}